\documentclass[10pt,twocolumn,letterpaper]{article}

\usepackage{cvpr}

\usepackage{times}
\usepackage{epsfig}
\usepackage{graphicx}
\usepackage{amsmath}
\usepackage{amssymb}

\usepackage{times}
\usepackage{soul}
\usepackage{url}
\usepackage[utf8]{inputenc}
\usepackage[small]{caption}
\usepackage{graphicx}
\usepackage{amsmath}
\usepackage{booktabs}
\usepackage{algorithm}
\usepackage{algorithmic}
\usepackage{amsthm}
\newtheorem{theorem}{Theorem}
\newtheorem{proposition}{Proposition}
\usepackage{subfig}
\usepackage{xcolor}

\usepackage{tikz}
\usepackage{pgfplots}
\cvprfinalcopy % *** Uncomment this line for the final submission

\def\cvprPaperID{3} % *** Enter the CVPR Paper ID here

\ifcvprfinal\fi

\begin{document}

%%%%%%%%% TITLE
\title{Revisiting the Evaluation of Uncertainty Estimation and Its Application to Explore Model Complexity-Uncertainty Trade-Off}

\author{Yukun Ding$^{1}$, Jinglan Liu$^{1}$, Jinjun Xiong$^{2}$, Yiyu Shi$^{1}$\\
$^{1}$ University of Notre Dame\\
$^{2}$ IBM Thomas J. Watson Research Center\\
{\tt\small \{yding5, jliu16, yshi4\}@nd.edu},
{\tt\small jinjun@us.ibm.com}
}

\maketitle
\thispagestyle{empty}

%%%%%%%%% ABSTRACT
\begin{abstract}
Accurately estimating uncertainties in neural network predictions is of great importance in building trusted DNNs-based models, and there is an increasing interest in providing accurate uncertainty estimation on many tasks, such as security cameras and autonomous driving vehicles. In this paper, we focus on the two main use cases of uncertainty estimation, {\em i.e.}, selective prediction and confidence calibration. We first reveal potential issues of commonly used quality metrics for uncertainty estimation in both use cases, and propose our new metrics to mitigate them. We then apply these new metrics to explore the trade-off between model complexity and uncertainty estimation quality, a critically missing work in the literature. Our empirical experiment results validate the superiority of the proposed metrics, and some interesting trends about the complexity-uncertainty trade-off are observed. 
\end{abstract}

\section{Introduction}
Deep neural networks (DNNs) have been widely used in vision tasks and achieved remarkable performance improvement. A major challenge in adopting DNNs to real-world mission-critical applications such as the medical image segmentation, is the lack of self-awareness and the tendency to fail silently~\cite{holzinger2017we}. 
In contrast, human's awareness of prediction uncertainty enables,
for example, human radiologists to conduct further investigations whenever they are in doubt for a diagnosis based on computed tomography (CT) images, and human drivers to slow down whenever they cannot clearly recognize an object. 
In order for DNNs to gain human's trust in making critical
decisions, especially in mission-critical scenarios, we need to equip DNNs with self-awareness on a par with its task competency. Most recently, much effort has been devoted to providing an accurate quantified score representing the uncertainty of every prediction, where wrongly predicted instances are expected to be assigned with low confidence scores and correctly predicted ones are expected to be assigned with high confidence scores \footnote{Confidence is the additive inverse of uncertainty with respect to 1, so they are used interchangeably in the literature.}~\cite{devries2018learning,guo2017calibration,malinin2018predictive,sakaridis2019semantic,shrikumar2019calibration}. 

The competency awareness of DNNs is commonly realized in two use cases of uncertainty estimation: {\em selective prediction}~\cite{geifman2017selective,lakshminarayanan2017simple,mandelbaum2017distance,nair2018exploring} and {\em confidence calibration}~\cite{guo2017calibration,heo2018uncertainty,kumar2018trainable,naeini2015obtaining,sander2019towards,seo2018confidence}. For the selective prediction, the obtained confidence scores are thresholded and the model can abstain from making predictions on samples with low confidence scores to achieve higher accuracy on the remaining part~\cite{hendrycks2016baseline}. For instance, in automatic segmentation of medical images, it is desired that the machine segments the common and easy area of medical images and refers the area with unusual appearance to the radiologists to ensure an extremely high accuracy~\cite{sander2019towards}. In this case, the confidence score is expected to be used for separating correct predictions and wrong predictions, and the popular quality metrics used to evaluate uncertainty estimation are {\em Area Under Receiver Operating Characteristic curve} (AUROC) and {\em Area Under Precision-Recall curve} (AUPR)~\cite{chen2018confidence,hendrycks2016baseline,malinin2018predictive}. For the confidence calibration, the aim is to provide a confidence score that approximates the empirical probability of a prediction being correct~\cite{guo2017calibration,sander2019towards}. 
For instance, in autonomous driving, human intervention is often not available in a timely manner and the high-level planning module will need such calibrated confidence score of pedestrian detection for instant decision making. In this case, common quality metrics are {\em Expected Calibration Error} (ECE) and {\em Maximum Calibration Error} (MCE) ~\cite{guo2017calibration,kumar2018trainable,sander2019towards}.

Despite the recent advancements, we show that the quality metrics of neural network uncertainty estimation used by most existing works could be problematic, potentially leading to unfair comparisons, confusing results, and/or undesired learning behaviors. Specifically, for the selective prediction, we show that even minimal changes in prediction models can make the commonly used evaluation based on AUROC and AUPR meaningless or even misleading. To address this issue, we propose to use
a different metric, called {\em Area Under Risk-Coverage} (AURC) curve as the primary metric for selective prediction. We show that AURC is the only reliable metric among AUROC, AUPR, and AURC, when the underlying prediction model changes, and is consistent with AUROC and AUPR when the underlying prediction model stays the same.
As for confidence calibration, we show that, because of the basic binning strategy employed, the commonly used evaluation metrics ECE and MCE cannot expose some large calibration error even in the high confidence area. Moreover, they are vulnerable to internal compensation and inaccurate accuracy estimation in each confidence interval, which leads to poor robustness and inferior accuracy. Therefore, we propose a new binning strategy, called {\em adaptive binning}, for the evaluation by ECE and MCE, and empirically show its superiority.

While the complexity-accuracy trade-off of DNN-based models has been extensively studied in the literature, the effect of model complexity on uncertainty estimation quality is almost unknown. However, with the prevalence of DNNs in real-world applications, it is ever more important for model designers to seek the best trade-off between cost and different aspects of model performance under various resources constraints. Therefore, a better understanding of the uncertainty-related performance changes with model complexity is required. We first give some theoretical analysis of the relation between the selective prediction and the confidence calibration. Then we use our proposed new metrics to explore the effect of model complexity on the uncertainty-related model performance. Our study serves two purposes. First, it validates the effectiveness and robustness of our proposed evaluation metrics. Second, it provides the first empirical study on how uncertainty-related model performance changes with model complexity. From our study, we observe that, interestingly, estimation quality changes significantly with model complexity for selective prediction, but is insensitive to model complexity for confidence calibration. 

In summary, the main contributions of this paper are as follows:
\begin{itemize}
\item We identify the potential issues of commonly used quality metrics for uncertainty estimation in both selective prediction and confidence calibration, and propose new metrics that provide more reliable and informative evaluations.

\item As an application and validation of the proposed metrics, we provide the first exploration of complexity-uncertainty trade-off, and show some interesting observations. 
\end{itemize}

\section{Related Works}
\label{sec:related}
\textbf{Uncertainty Estimation.} \quad Various methods exist in the literature to estimate the uncertainty of neural network predictions~\cite{devries2018leveraging,geifman2018bias,heo2018uncertainty,kendall2017uncertainties,liang2017enhancing}. The most popular approaches include softmax probability~\cite{hendrycks2016baseline,guo2017calibration}, Monte Carlo dropout~\cite{gal2016dropout,roy2019bayesian}, and learned confidence estimation~\cite{devries2018learning,liu2019deep}. The uncertainty estimation can either explicitly affect the model during the training process~\cite{dhamija2018reducing,kumar2018trainable,lee2017training} or work as a post-processing step that does not affect the underlying prediction models~\cite{chen2018confidence,guo2017calibration}. Note that in our definition and analysis, we did not make any assumption on how the confidence score is obtained or any correlation between the prediction and the confidence score. They can be obtained by any prediction model and any uncertainty estimation method.

\textbf{Evaluation Methods.} \quad The commonly used evaluation metrics for selective prediction are AUROC and  AUPR~\cite{chen2018confidence,hendrycks2016baseline,malinin2018predictive}. Recently, E-AURC has been used to evaluate the uncertainty estimation quality in a selective prediction scenario~\cite{geifman2018bias}. However, the E-AURC has exactly the same problem with AUROC and AUPR, because the accuracy difference is not considered, which is detailed in Section~\ref{sec:selective}. It is meaningless when the underlying prediction model varies, which happens in many cases, \eg comparing different models, evaluating algorithms that alter the training process~\cite{mandelbaum2017distance}, and when either MC-dropout or ensemble is used for uncertainty estimation~\cite{lakshminarayanan2017simple,nair2018exploring}.

Currently, the confidence calibration quality of neural network-based models are evaluated by ECE and MCE~\cite{guo2017calibration,kumar2018trainable,naeini2015obtaining,sander2019towards,seo2018confidence}.  In order to minimize the ECE, a differentiable proxy to ECE named MMCE is used for calibration-aware network training~\cite{kumar2018trainable}. Negative Log Likelihood (NLL) and Brier Score are used as indirect and supplementary measurements in some works~\cite{guo2017calibration,seo2018confidence}. We note that Brier Score and NLL are not suitable as primary metrics for confidence calibration, because they prefer better prediction rather than better calibration by design. ECE and MCE have been extended from the binary setting to the multi-class setting in ~\cite{nixon2019measuring,vaicenavicius2019evaluating}. In either way, the computation of ECE and MCE heavily depends on the binning strategy which is the focus of this work. Equal-size binning where every bin has a same number of samples was proposed as a remedy for the known issues of the common fixed equal-size binning ~\cite{nixon2019measuring,vaicenavicius2019evaluating}. However, we show that although equal-size binning helps, it is still not flexible enough to deal with highly non-uniform confidence distribution. ~\cite{naeini2015obtaining} uses a Bayesian score to average a number of models with equal-size binning. Such modeling averaging is orthogonal to our binning method. In addition, it is used as a calibration method to improve the performance measured by ECE that uses the conventional
equal-range binning.

\section{Problem Setting}
\label{sec:setting}

We put our discussion in a general classification setting. See ~\cite{kuleshov2018accurate} for recent advance in the regression setting. Following~\cite{kumar2018trainable}, we denote $\mathcal{Y}=\{1,2,\dots,K\}$ as the set of class labels, $\mathcal{X}$ as the input space, $\mathcal{D}$ as the data distribution, and $N_{\theta}(y|x)$ as the probability distribution of model predictions with input $x$, and model parameters $\theta$. For each input sample $x_i$ and true label $y_i$, the model gets a predicted label $\hat{y}_i=\text{argmax}_{y\in\mathcal{Y}}N_{\theta}(y|x_i)$ and a confidence score $r_i$. If $\hat{y}_i=y_i$, which means the prediction is correct, we have the correctness score $c_i=1$. Otherwise, $c_i=0$.
Then the distribution over $r$ and $c$ on $\mathcal{D}$ can be denoted as $P_{\theta,\mathcal{D}}(r,c)$.

\textbf{Selective Prediction.} \quad In selective prediction, with a confidence score $r_i$ for each input $x_i$ and a threshold $t$, the input from dataset $X$ and the prediction $\hat{Y}$ are split to $X_h=\{x_i|r_i>=t\}$, $X_l=\{x_i|r_i<t\}$ and $\hat{Y_h}=\{\hat{y_i}|r_i>=t\}$, $\hat{Y_l}=\{\hat{y_i}|r_i<t\}$ respectively. The model abstains from making prediction on $X_l$. Ideally, $\hat{Y_l}$ contains all wrong predictions and $\hat{Y_h}$ contains all correct predictions so that the error is avoided with the minimal cost. In this case, $r_i$ is used for separating correct predictions and wrong predictions, which is a binary classification problem and therefore the common quality metrics are
AUROC and AUPR~\cite{chen2018confidence,hendrycks2016baseline,malinin2018predictive,mandelbaum2017distance}. 

\textbf{Confidence Calibration.} \quad Confidence calibration aims to give a confidence score $r\in [0,1]$ that directly reflects the probability of the prediction being correct. The difference between the probability of correct prediction $E_{P_{\theta,\mathcal{D}}(c|r)}[c]$ and the confidence score $r$ is defined as the calibration error. Consequently, the expected calibration error (ECE) and maximum calibration error (MCE) are defined as: 
\begin{align}
    \text{ECE}(P_{\theta,\mathcal{D}}) =E_{P_{\theta,\mathcal{D}}(r)}[|E_{P_{\theta,\mathcal{D}}(c|r)}[c]-r|]
\end{align}%
\begin{align}
    \text{MCE}(P_{\theta,\mathcal{D}}) =\max_{r\in[0,1]}|E_{P_{\theta,\mathcal{D}}(c|r)}[c]-r|
\end{align}%

Practically, given a finite number of samples in $D\sim P_{\theta,\mathcal{D}}$, ECE and MCE are calculated by partitioning the $[0,1]$ range to $n$ bins according to a binning strategy. For every bin, an average accuracy and an average confidence are calculated using all samples inside. The difference between the average accuracy and the average confidence is the calibration error, which is denoted as calibration gap in~\cite{guo2017calibration}. The standard practice is to use $n$ equal-range bins where $n$ is chosen as 10 in the literature~\cite{guo2017calibration,heo2018uncertainty,kumar2018trainable,naeini2015obtaining,sander2019towards,seo2018confidence}. Specifically, the partition is defined as $B_j=[\frac{j-1}{n},\frac{j}{n}]$, $j=\{1,\dots,n\}$. It is possible to use different binning strategy such that these bins are not uniformly distributed and the definition of $B_j$ will change accordingly. Given $D_j=\{x_i|r_i\in B_j\}$, ECE and MCE are computed as $\hat{\text{ECE}}$ and $\hat{\text{MCE}}$ by:
%\small
\begin{align}
    \hat{\text{ECE}}(P_{\theta,\mathcal{D}}) =\frac{1}{|D|}\sum_{j=1}^{n}|\sum_{x_i\in D_j}c_i-\sum_{x_i\in D_j}r_i|
\end{align}%
\begin{align}
    \hat{\text{MCE}}(P_{\theta,\mathcal{D}}) =\max\frac{1}{|D_j|}|\sum_{x_i\in D_j}c_i-\sum_{x_i\in D_j}r_i|
\end{align}%
%\normalsize

It is worth mentioning that both ECE and MCE are proper scoring rules but not strictly proper scoring rules ~\cite{gneiting2007strictly}. However, even strictly proper scoring rules do not guarantee a reliable evaluation and comparison~\cite{merkle2013choosing}. Their potential issues are discussed in Section~\ref{sec:solution}.

In addition to the quantified metric, Reliability Diagram ~\cite{guo2017calibration,kumar2018trainable,sander2019towards,seo2018confidence} is used as a standard qualitative analysis tool in the literature. It plots the empirical accuracy in each bin and the calibration error. Such a diagram not only visualizes the calibration error at different confidence intervals but also shows how the ECE and the MCE are calculated.

\section{Evaluation Metrics: Issues \& Solutions}
\label{sec:solution}
In this section, we discuss some issues of the existing quality metrics for uncertainty estimation that may lead to unfair comparison or neglected problems, and then present new metrics to mitigate them. In all figure captions and tables, $\uparrow$ means the higher the better, and $\downarrow$ means the lower the better. 

\subsection{Selective Prediction}
\label{sec:selective}
 We remark that comparing AUROC and AUPR is fair only when the underlying prediction models are the same. There was no proper treatment used in the literature when comparing uncertainty estimation methods with different prediction models as shown in recent works~\cite{geifman2018bias,malinin2018predictive,mandelbaum2017distance}. Below we use an example to show that even a small difference in the underlying prediction model could make the comparison of AUROC and AUPR meaningless if not misleading. Then we discuss the advantage of the proposed AURC over AUROC and AUPR.  

In order to know the relative performance of models as a priori, we build an illustrative example based on a real-world network. We train a 100-layer DenseNet on Cifar10 with standard settings and denote it as DenseNet. For a given test dataset $X$, we further define $X_c$ and $X_p$ as $X_c = \{x_i| c_i=1\}$ and $X_p = \{x_i| r_i<t_m, x_i \in X_c\}$ where $t_m$ is a threshold such that $m=|X_p|$. Consider there is a network named DenseNet-m that makes the same predictions with DenseNet for all samples in $X$ except for that in $X_p$ and DenseNet-m has $c_i=0$ for all $x_i\in X_p$. In other words, DenseNet-m is the same with DenseNet except that DenseNet-m makes $m$ more wrong predictions in the most uncertain predictions out of the total $10^4$ samples. By varying the value of $m$, we get different DenseNet-m. Note that DenseNet is equivalent to DenseNet-0. With this setup, DenseNet-m with smaller $m$ has equal or higher accuracy than that with bigger $m$ at any threshold $t$ in selective prediction. Therefore, it is convincing to conclude that bigger $m$ indicates worse selective prediction quality. A proper evaluation metric is expected to correctly reflect the relative performance of different variants of DenseNet-m.

We plot ROC curves and PR curves of DenseNet-m with different values of $m$ in Figure~\ref{fig:CurveCompare}. Remember that both ROC curves and PR curves are the higher the better. Both curves suggest a questionable result that DenseNet-m with a bigger $m$ is better. The reason is that both AUROC and AUPR only measure a model's ability to distinguish correct and wrong predictions while assuming the numbers of correct and wrong predictions are the same.
An accuracy change of 0.2\% (comparing DenseNet-0 and DenseNet-20) could give a significant improvement on AUPR while the actual performance is getting worse.
Therefore, when the underlying prediction models are different, AUROC and AUPR fail to correctly reflect the model's actual performance change. 

%\vspace{-0.05in}
\begin{figure}[htb]
\centering
%\hspace*{-0.5em}
\subfloat[ROC curve$\uparrow$]{
\begin{minipage}[b]{0.33\linewidth}
\label{fig:CurveCompare1}
\includegraphics[height=0.8in]{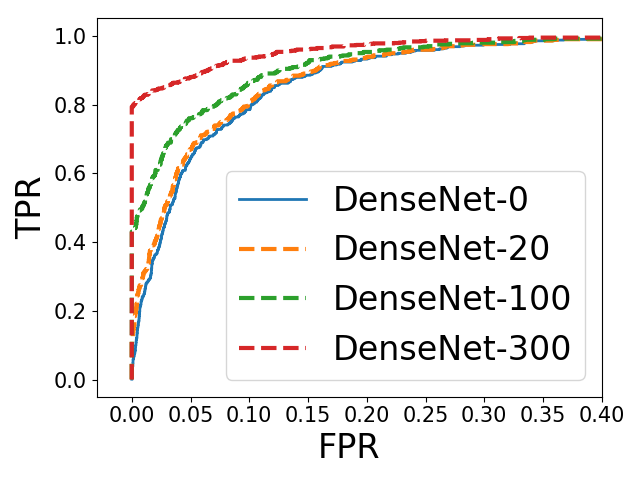}
\end{minipage}
}
\subfloat[PR curve$\uparrow$]{
\begin{minipage}[b]{0.33\linewidth}
\label{fig:CurveCompare2}
\includegraphics[height=0.8in]{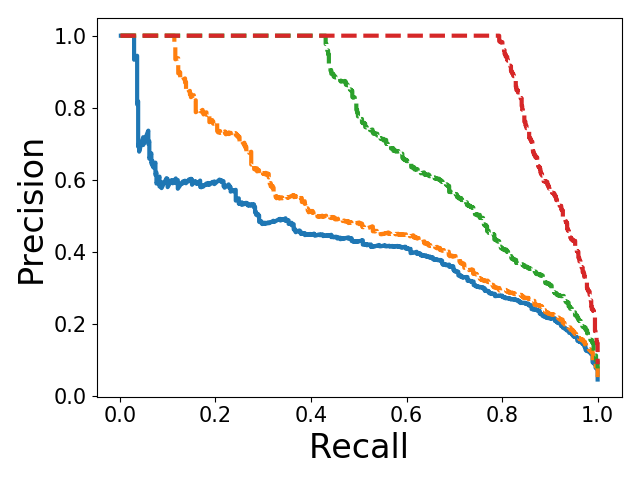}
\end{minipage}
}
\subfloat[RC curve $\downarrow$]{
\begin{minipage}[b]{0.33\linewidth}
\label{fig:CurveCompare3}
\includegraphics[height=0.8in]{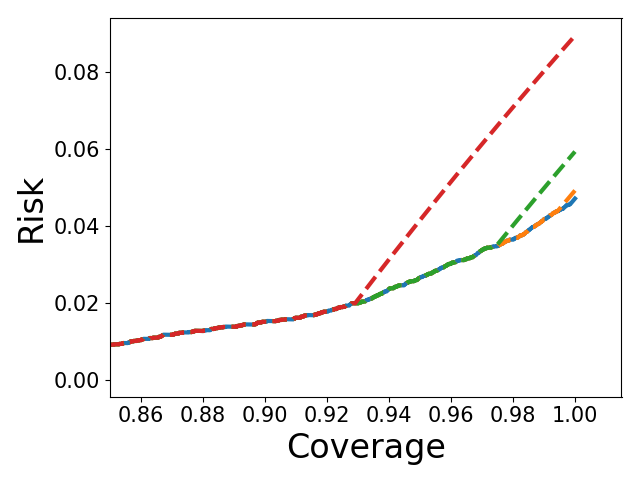}
\end{minipage}
}
%\quad
\caption{Evaluation curves of DenseNet-m. We term the misclassified/classified samples as positive/negative samples following the literature. }
\label{fig:CurveCompare}
%\vspace{-0.05in}
\end{figure}

The observation is also confirmed in quantitative results. In Table~\ref{tab:CurveCompare}, the quantitative results of both AUROC and AUPR suggest that DenseNet-m with bigger $m$ performs better which contradicts the prior knowledge.

\begin{table}[htb]
\caption{Quantitative comparison of AUROC, AUPR and AURC. Similar to the results in the literature, AUROC's change is relatively small because of the class imbalance. }
\centering
\begin{tabular}{lccccc}  
\toprule
Model  & Acc.$\uparrow$ & AUROC$\uparrow$  & AUPR$\uparrow$ & AURC$\downarrow$ \\
\midrule
DenseNet-0       & 95.30& 93.71  &  43.36 & 0.438     \\
DenseNet-20            & 95.10& 94.17  & 52.36  & 0.439  \\
DenseNet-100          & 94.30& 95.80  & 72.80  & 0.456  \\
DenseNet-300            & 92.30& 98.08  & 91.64  & 0.605  \\
\bottomrule
\end{tabular}
\label{tab:CurveCompare}
%\vspace{-0.15in}
\end{table}

In summary, evaluating models with AUPR and AUROC not only fails to provide a fair comparison, but also implicitly encourages the bad practice of reducing model accuracy in designing uncertainty estimation methods.

To mitigate this issue, we propose to do the evaluation with the Risk-Coverage (RC) curve instead. The coverage denotes the percentage of the input processed by the model without human intervention and the risk denotes the level of risk of these model prediction. Formally,
\begin{align}
    coverage=\frac{|X_h|}{|X|}\\
%\end{align}
%\begin{align}
    risk = \mathcal{L}(\hat{Y_h})
\end{align}
%\normalsize
where $\mathcal{L}$ is a loss function measuring the prediction quality. For classification, the 0/1 loss is commonly used~\cite{geifman2017selective} as it measures the classification accuracy.

The risk-coverage curve reflects the nature of the selective prediction very well by definition as the motivation of the selective prediction is to reduce the coverage of the model in order to achieve higher accuracy. 

 The RC curves of DenseNet-m are shown in Figure~\ref{fig:CurveCompare3}. Quantitative comparison of AUROC, AUPR, and AURC are shown in Table~\ref{tab:CurveCompare}. Note that although the RC curve has been used in the literature to demonstrate selective classification~\cite{geifman2017selective}, using AURC as the evaluation metric of uncertainty estimation for selective prediction is first proposed in this work. In both qualitative and quantitative results, only AURC gives the correct performance ranking which is in line with the prior knowledge. The reason is that AURC by definition naturally evaluates the combined results of prediction and uncertainty estimation without the assumption that the prediction models are the same. Without AURC, even if the accuracy, AUROC, and AUPR are reported together, it is still unknown that which model in Table~\ref{tab:CurveCompare} performs the best in selective prediction. Therefore, AURC provides the only reliable evaluation when the underlying prediction models are different.

We further show that AURC is still a good metric when the  accuracy of underlying prediction models are the same, because it correctly recognizes the better model just like AUROC and AUPR. The well-known connection and consistency between ROC curve and PR curve are established in ~\cite{davis2006relationship} by proving that the curve of one model dominates the curve of another model in ROC space, if and only if it also dominates the other in the PR space. In this paper, we show that the RC curve shares the same inherent connection with ROC curve and PR curve by giving Theorem~\ref{T2}. The proof is given in the supplementary.

\begin{theorem}
\label{T2}
  For any two models A and B of the same accuracy and their uncertainties measured by arbitrary methods (which can be different for A and B), the curve of A dominates that of B in the ROC space, if and only if the curve of A dominates that of B in the Risk-Coverage space. \end{theorem}

 Another issue with the evaluation practice we identified is the poor generality out of classification tasks. For example, the image segmentation quality cannot be properly evaluated by pixel-wise accuracy. Consequently, even if the underlying prediction models are the same, AUROC and AUPR still fail to accurately reflect the performance of selective prediction. In contrast, the AURC can be easily extended to this case by using a suitable $\mathcal{L}$ for domain-specific performance measure. For example, when measuring the image segmentation quality, the risk can be defined as $1-Dice$ where $Dice$ is a commonly used quality metric for image segmentation.

In selective prediction, each low confidence prediction leads to a special process such as processing by a bigger model, examining by human experts, or making conservative decisions by the controller. Such operation is usually ``expensive'' and thus the coverage directly determines the overall operation cost that is an important metric in a selective prediction application. However, this metric is not available from the conventional ROC curve or PR curve. In contrast, RC curve merges two accuracy axes to one and adds the cost axis to show the cost-performance trade-off which arguably makes it easier for human administrators to choose an operating point.

 In summary, using AURC as a primary evaluation metric has the following advantage. (i) when the underlying prediction models are the same, AURC is an effective quality metric to indicate the performance of selective prediction; (ii) when the prediction models are different which happens a lot in the literature due to the emerging trend of uncertainty-aware training~\cite{mandelbaum2017distance,geifman2018bias,malinin2018predictive,kumar2018trainable}, using AURC instead of AUPR and AUROC prevents unfair and potential misleading comparison. (iii) AURC can be generalized to distinct tasks with task-specific evaluation metrics while AUPR and AUROC cannot. (iv) AURC is an alternative optimization objective to directly maximize the performance of selective prediction and helps to avoid weighing multiple objective terms in related work~\cite{mandelbaum2017distance,malinin2018predictive,kumar2018trainable}. (v) AURC directly shows the cost-performance trade-off in selective prediction which is not visible in the conventional ROC curve or PR curve.

\subsection{Confidence Calibration}
\label{sec:calib}
The accuracy and reliability of Reliability Diagrams, ECE, and MCE highly depend on the underlying binning strategy, whose limitations are explained as below. 

\textbf{Undetectable Error.} \quad We use the same original DenseNet used above as a real-world example to show the problem of the commonly used ECE, MCE, and Reliability Diagrams. The maximum softmax probability~\cite{hendrycks2016baseline} is used as the confidence score. As shown in Figure~\ref{fig:RD1}, the calibration error on $[0.9,1]$ is as small as 0.0215 which means the average error between the confidence and accuracy is 2.15\%. One would expect that for confidence in this interval, the accuracy is very close to the confidence. However, as shown in Figure~\ref{fig:RD2}, for input samples with confidence in $[0.9,0.91]$, the accuracy is only 50\%. For input samples with confidence in $[0.9,0.96]$, the accuracy is lower than 73\%. This problem can mainly be attributed to the large bin range and the highly non-uniform distribution of confidence as shown in Figure~\ref{fig:RD1}. 

Most samples have a confidence score in $[0.98,1.0]$ and the calibration error on that range is small. Then an average view makes the high calibration error on $[0.9,0.96]$ undetectable by normal Reliability Diagrams, ECE, and MCE. Note that the big calibration error on $[0.9,0.96]$ by no means should be tolerated, because it has a higher sample density than all nine bins on its left and ignoring it may jeopardize mission-critical systems.

%\vspace{-0.05in}
\begin{figure}[htb]
\vspace{-0.05in}
\centering
%\hspace*{\fill}
\subfloat[Normal setting with 10 bins]{
\label{fig:RD1}
\begin{minipage}[b]{0.5\linewidth}
\centering
\includegraphics[height=1.2in]{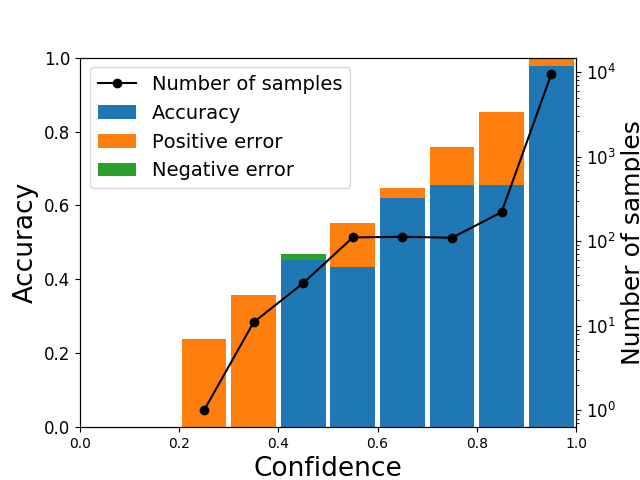}
\end{minipage}
}
\subfloat[Diagrams on \ensuremath [$0.9,1$\ensuremath {]} ]{
\begin{minipage}[b]{0.5\linewidth}
\label{fig:RD2}
\centering
\includegraphics[height=1.2in]{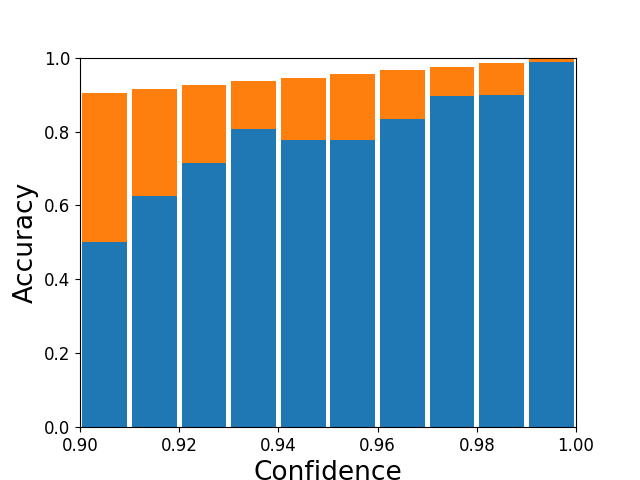}
\end{minipage}
}
%\hspace*{\fill}
\caption{Undetectable error in Reliability Diagrams. In all Reliability Diagrams, positive error means confidence is larger than accuracy.}
\label{fig:RD}
\end{figure}

\textbf{Internal Compensation.}  \quad Even if the confidence distribution is relatively uniform, we remark that ``internal compensation'' can happen inside a bin and the ECE obtained is overly optimistic. As can be seen from Figure~\ref{fig:RD1}, the error is not always positive or negative. In fact, the different sign of the error also exists inside a bin. This makes the computed ECE lower than a more accurate one computed based on a higher resolution. We conclude this effect in Proposition~\ref{T1} and the complete proof is provided in the supplementary.

\begin{proposition}
\label{T1}
  For any bin selection, $\hat{\text{ECE}}(P_{\theta,\mathcal{D}})=\text{ECE}(P_{\theta,\mathcal{D}})$ if and only if for any bin $B_j$,  $E_{P_{\theta,\mathcal{D}}(c|r_k)}[c]\geq r_k$ for all $r_k\in B_j$ or $E_{P_{\theta,\mathcal{D}}(c|r_k)}[c]\leq r_k$ for all $r_k\in B_j$. Otherwise, $\hat{\text{ECE}}(P_{\theta,\mathcal{D}})<\text{ECE}(P_{\theta,\mathcal{D}})$.
  \end{proposition}
With the assumption that $E_{P_{\theta,\mathcal{D}}(c|r_k)}[c]$ is available, the range of bins should be as small as possible to recover the actual ECE. However, $E_{P_{\theta,\mathcal{D}}(c|r_k)}[c]$ is only available with enough samples which leads to the problem of inaccurate accuracy estimation.

\textbf{Inaccurate Accuracy Estimation.} \quad
A relatively straightforward solution to the aforementioned problems is to increase the number of bins to get higher resolution on the confidence scores. However, using more bins does not always get better results. Even though $\frac{1}{|D_j|}\sum_{x_i\in D_j}c_i$ is an unbiased and consistent estimator of $E_{P_{\theta,\mathcal{D}}(c|r)}[c]$ for $x_i\in D_j$, with more bins $|D_j|$ becomes smaller and the inaccurate approximation of $E_{P_{\theta,\mathcal{D}}(c|r)}[c]$ leads to inaccurate ECE. 
In fact, another loophole of the Reliability Diagrams and MCE is that, $\frac{1}{|D_j|}\sum_{x_i\in D_j}c_i$ may not provide accurate estimation for $E_{P_{\theta,\mathcal{D}}(c|r)}[c]$ when $|D_j|$ is small. A real-world example is shown in Figure~\ref{fig:RDC1}. In order to make the low accuracy around 0.95 visible, the number of bins has to be increased from 10 to 50. The significant fluctuation is a result of the inaccurate accuracy when the number of samples in the bin is small.

We remark that although the inaccurate accuracy estimation can be solved if there are excessive samples, it cannot fix the undetectable error and internal compensation,  because the underlying reason is the highly nonuniform confidence distribution instead of limited samples. A seemingly feasible remedy is to use equal-size binning~\cite{nixon2019measuring,vaicenavicius2019evaluating}, where each bin has the same number of samples, instead of equal-range binning.  However, when the bin lies in a confidence region where samples are sparse, the resulting confidence range may be too large and less informative. On the other hand, when it lies in a region where samples are dense, the range of bins becomes too small and accuracy estimation is suboptimal. This can be seen from Figure~\ref{fig:RDC2} and Figure~\ref{fig:RDC3}, when 50 and 100 bins are used respectively. In both Figure~\ref{fig:RDC2} and Figure~\ref{fig:RDC3}, it can be seen that a very wide bin exists in the low confidence region while an excessive number of bins reside in high confidence region.

To tackle this challenge, we resort to an adaptive binning strategy, where the number of samples in a bin is adaptive to the distribution of the samples in the confidence range. We achieve a dynamic balance between the resolution and the accuracy estimation by associating the number of samples in each bin with the range of the bin. In this way, more samples can be included for better accuracy estimation when the samples are dense and fewer samples will be used to avoid too large range when the samples are sparse.
Specifically, we use $n=0.25\left(\frac{Z_{\alpha/2}}{\epsilon}\right)^2$ to estimate the number of samples needed to estimate the accuracy for each bin where $\epsilon$ is the error margin, $Z_{\alpha/2}$ is the $Z$-score of a standard normal distribution and $1-\alpha$ is the confidence interval. Even though there are still two hyper-parameters, we find that the result is not sensitive to these parameters in a wide range due to its high robustness. We use an 80\% confidence interval and let $\epsilon$ equal to the width of the confidence range of the bin in all experiments. We denote the resulting new adaptive metrics as AECE and AMCE in the rest of this paper. To make it easy for other researchers to use this adaptive binning, we provide the details and our implementation as an open-source tool at \url{https://github.com/yding5/AdaptiveBinning}. The computation overhead is minimal and the complexity is still $\mathcal{O}(|D|)$, so there is no scalability issue. The result on the same network is shown in Figure~\ref{fig:RDC4}, which achieves the best results in terms of capturing the calibration error and alleviating all the aforementioned problems. More examples of adaptive binning are shown in the supplementary. We hope researchers can consider to use this new quality metric in confidence calibration in the future.

\begin{figure}[htb]
\vspace{-0.05in}
\centering
%\hspace*{-2.0em}
\subfloat[50 equal-range bins]{
\begin{minipage}[b]{0.5\linewidth}
\label{fig:RDC1}
\includegraphics[height=1.2in]{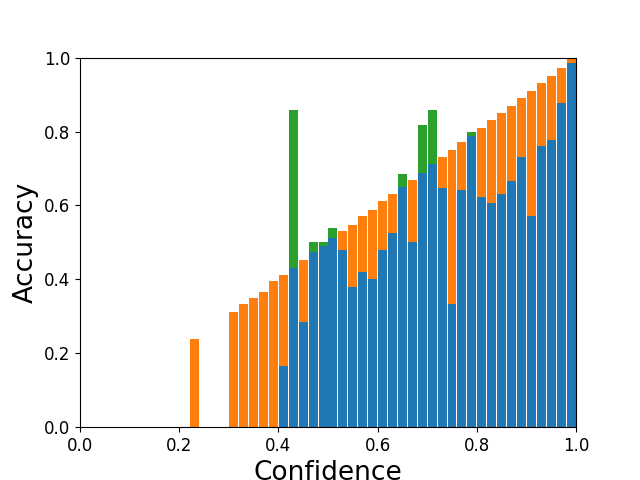}
\end{minipage}
}%\hspace*{-0.3em}
\subfloat[50 equal-size bins]{
\begin{minipage}[b]{0.5\linewidth}
\label{fig:RDC2}
\includegraphics[height=1.2in]{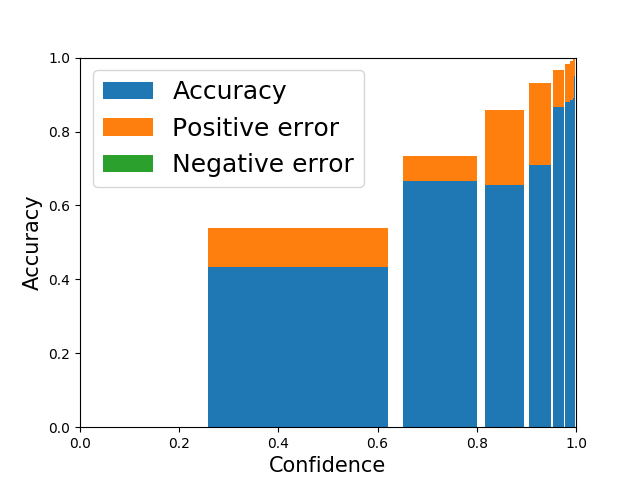}
\end{minipage}
}\\
%\hspace*{-0.3em}
\subfloat[100 equal-size bins]{
\begin{minipage}[b]{0.5\linewidth}
\label{fig:RDC3}
\includegraphics[height=1.2in]{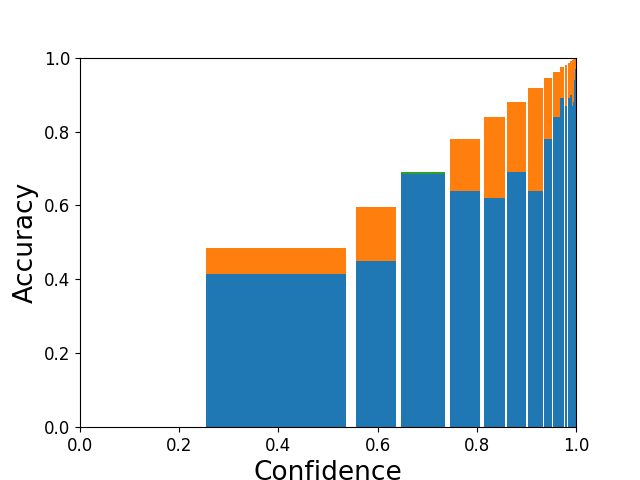}
\end{minipage}
}%\hspace*{-0.3em}
\subfloat[Adaptive binning]{
\begin{minipage}[b]{0.5\linewidth}
\label{fig:RDC4}
\includegraphics[height=1.2in]{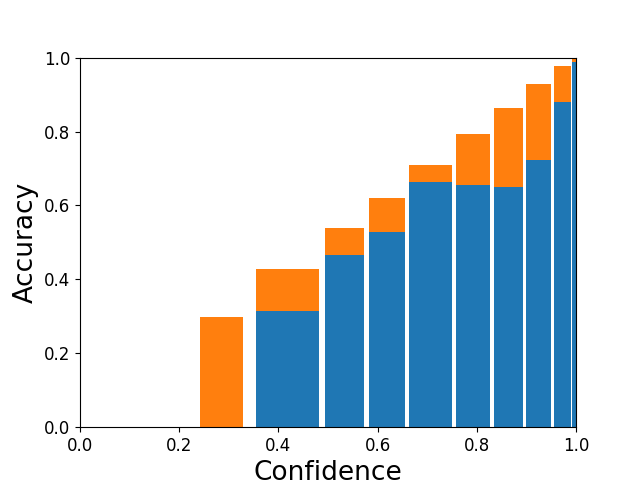}
\end{minipage}
}
%\quad
\caption{Reliability Diagrams of various binning methods.}
\label{fig:RDC}
\vspace{-0.04in}
\end{figure}

Theoretically, the issues discussed above can also happen in other probabilistic forecast problems, \eg weather forecast~\cite{lai2011evaluating} where similar evaluation metrics such as ECE and Brier Score are used~\cite{jolliffe2012forecast,zamo2018estimation}. Furthermore, consistency bar is used to partially solve the inaccurate accuracy estimation issue by indicating the fluctuations of the observed frequencies (accuracy in our context) caused by the limited samples in each bin with the \textit{consistency resampling} technique~\cite{bickel2008verification}. However, consistency bar can only indicate a reasonable fluctuation range of the frequencies that a reliable calibration would likely fall into. It does not change the binning and cannot help with the poor estimation quality when the number of samples is small or the range is not appropriate. The reason that these solutions cannot fully solve the problem has two parts. First, the neural network uncertainty can be more non-uniformly distributed compared with the weather forecast, especially when the weather forecast confidence scores are clustered to 11 uniformly distributed options in $[0,1]$~\cite{brocker2007increasing}. Second, uncertainty for neural networks are more critical, especially in high confidence area \eg the difference between the accuracy of 0.95 and 0.99 can be much more significant in pedestrian detection in an autonomous vehicle than that in the weather forecast.% discrete

\section{Effect of Model Complexity on Uncertainty Estimation}
\label{sec:exp}
In this section, we apply the proposed evaluation metrics in a series of experiments to validate the effectiveness and robustness of our proposed evaluation metrics and provide the first empirical study on how uncertainty-related model performance is affected by model complexity.  

\subsection{Relation Between the Two Use Cases}
\label{sec:exp1}
Before delving into the effect of model complexity on selective prediction and confidence calibration, it is interesting to analyze whether the effect would be similar for the two cases, which may help to justify the different trends observed from the experiments in Section~\ref{sec:exp2}.

When the prediction model is given, the performance of selective prediction is solely determined by the relative ranking of $r_i$ and $r_j$ for $c_i=1$ and $c_j=0$. The specific values of $r$ do not matter due to the thresholding mechanism. Even though a good threshold may be unknown, existing statistical methods are available for finding a desirable threshold~\cite{geifman2017selective}. In contrast, for confidence calibration, the specific value of $r$ does matter. The quality of the given confidence score is evaluated based on the difference between the confidence score and the expected accuracy of the samples with this score. As such, we expect that \textbf{performance of selective prediction and confidence calibration are not necessarily correlated.} One can construct illustrative examples to show that the confidence score $r$ can be perfect for one case but bad for the other. When $r_i=0.5$ for all $x_i\in\mathcal{D}$ and $\frac{\sum _{x_i\in\mathcal{D}}c_i}{|\mathcal{D}|}=\frac{1}{2}$, we have $\text{ECE}(P_{\theta,\mathcal{D}})=0$ but selective prediction is not feasible. When $r\in\{0.9,1\}$, $c_i=1$ for all $ r_i=1$ and $c_j=0$ for all $ r_j=0.9$, selective prediction achieves the best possible result but $\text{ECE}(P_{\theta,\mathcal{D}})$ is as high as 0.9.

 Proposition~\ref{TU} further shows that a confidence estimation is perfect in both cases,  if and only if it perfectly knows the correctness for each prediction which is almost impossible in practice. The proof is given in the supplementary.
\begin{proposition}
\label{TU}
The uncertainty estimation $r$ is perfect for both selective prediction and confidence calibration, if and only if for all samples $r\in\{0,1\}$, $E_{P_{\theta,\mathcal{D}}(c|r=0)}[c]=0$, and $E_{P_{\theta,\mathcal{D}}(c|r=1)}[c]=1$.
  \end{proposition}
  The results imply that given limited learning capability, there exists a trade-off between two aspects of uncertainty estimation and the model should be optimized for the specific use case.

\subsection{Experiment Results}
\label{sec:exp2}
We evaluate the uncertainty estimation quality of a series of models for selective prediction and confidence calibration in image classification and medical image segmentation.
There are different uncertainty estimation methods available. In this work, we use maximum softmax probability~\cite{hendrycks2016baseline} and temperature scaling~\cite{guo2017calibration} for selective prediction and confidence calibration as they are popular and shown to be competitive with more complex approaches~\cite{snoek2019can,chen2018confidence,geifman2019selectivenet}. 

We first evaluate two popular networks DenseNet~\cite{huang2017densely} and WideResNet~\cite{zagoruyko2016wide} on Cifar10 and Cifar100~\cite{krizhevsky2009learning} to cover different levels of difficulty and accuracy. For DenseNet, we keep the growth rate at 12 and reduce its depth from 100 to 10. For WideResNet, we change the widen factor of a 16-layer network and a 28-layer network for a comparable number of parameters with DenseNet. 

\textbf{The performance of selective prediction increases with model size.} For selective prediction, we first show how the conventional AUPR metric changes with the model size, and the results are shown in Figure~\ref{fig:F_AUPR} and Figure~\ref{fig:F_AUPR_C10}. It is shown that AUPR decreases with the model size, indicating networks' decreasing capability to differentiate wrongly predicted samples and correctly prediction samples. The reason is that wrong predictions with high confidence scores, an issue known as over-confidence in high capacity neural networks~\cite{lee2017confident}, are usually caused by inherent learning limitation or data similarity instead of network capacity. As a result, although higher capacity models have fewer wrong predictions, they are increasingly concentrated in the high confidence area, which in turn makes accurate uncertainty estimation harder. However, this is against the intuition that larger networks learn probability distribution better and thus behave better in uncertainty estimation. The reason is that the impact of the original full-coverage accuracy in selective prediction is not taken into consideration. If we use the proposed AURC instead as shown in Figure~\ref{fig:F_AURC} and Figure~\ref{fig:F_AURC_C10}, the estimation quality becomes consistent with common expectation, showing the increased performance of selective prediction with increasing model size.

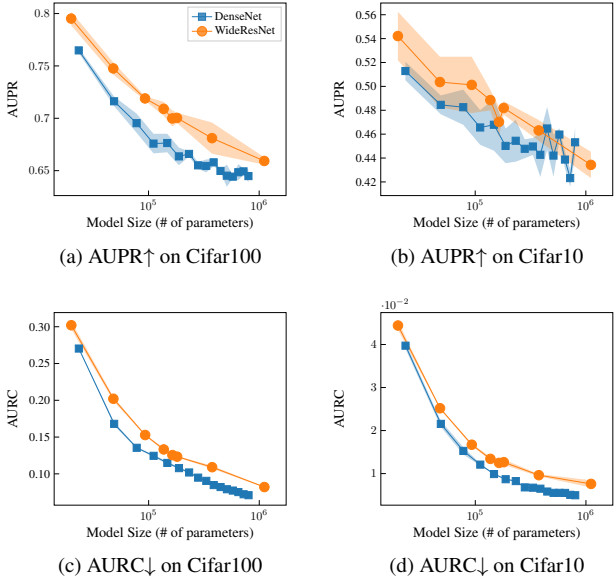
\begin{figure}[htb]
%\vspace{-0.15in}
\vspace{-0.05in}
\centering
%\hspace*{-2.2em}
\subfloat[AUPR$\uparrow$ on Cifar100]{
\begin{minipage}[b]{0.5\linewidth}
\label{fig:F_AUPR}
\scalebox{0.45}{% This file was created by tikzplotlib v0.8.7.
\begin{tikzpicture}
\pgfplotsset{every axis/.append style={
                    label style={font=\large}
                    }}
\definecolor{color0}{rgb}{0.12156862745098,0.466666666666667,0.705882352941177}
\definecolor{color1}{rgb}{1,0.498039215686275,0.0549019607843137}

\begin{axis}[
legend cell align={left},
legend style={fill opacity=0.8, draw opacity=1, text opacity=1, draw=white!80.0!black},
log basis x={10},
tick align=inside,
%tick pos=both,
tick pos=left,
x grid style={white!69.01960784313725!black},
xlabel={Model Size (\# of parameters)},
xmin=14000, xmax=1750000,
xmode=log,
xtick style={color=black},
xtick={1000,10000,100000,1000000,10000000,100000000},
xticklabels={\(\displaystyle {10^{3}}\),\(\displaystyle {10^{4}}\),\(\displaystyle {10^{5}}\),\(\displaystyle {10^{6}}\),\(\displaystyle {10^{7}}\),\(\displaystyle {10^{8}}\)},
y grid style={white!69.01960784313725!black},
ylabel={AUPR},
ymin=0.626189567772737, ymax=0.810066409886069,
ytick style={color=black}
]
\path [fill=color0, fill opacity=0.3]
(axis cs:800000,0.650799547949574)
--(axis cs:800000,0.638840991172102)
--(axis cs:721900,0.646201978820074)
--(axis cs:647500,0.64071135518452)
--(axis cs:576900,0.642417947073026)
--(axis cs:510000,0.634547606050615)
--(axis cs:447000,0.64889801912557)
--(axis cs:387700,0.654345385111464)
--(axis cs:332100,0.648358239734892)
--(axis cs:280300,0.651721226431571)
--(axis cs:232300,0.66348396573067)
--(axis cs:188100,0.655161085943575)
--(axis cs:147600,0.66752608006476)
--(axis cs:110900,0.666482531851527)
--(axis cs:78000,0.686163462071873)
--(axis cs:48800,0.711621353061349)
--(axis cs:23400,0.762485961875304)
--(axis cs:23400,0.76715182219476)
--(axis cs:23400,0.76715182219476)
--(axis cs:48800,0.720800479592504)
--(axis cs:78000,0.704490891637282)
--(axis cs:110900,0.685209173218087)
--(axis cs:147600,0.685214408441315)
--(axis cs:188100,0.672069305814229)
--(axis cs:232300,0.668471994222346)
--(axis cs:280300,0.658395557767777)
--(axis cs:332100,0.660653604988516)
--(axis cs:387700,0.661530708950965)
--(axis cs:447000,0.650555061911937)
--(axis cs:510000,0.655855580017163)
--(axis cs:576900,0.646114238142097)
--(axis cs:647500,0.656231338158039)
--(axis cs:721900,0.652414942392535)
--(axis cs:800000,0.650799547949574)
--cycle;

\path [fill=color1, fill opacity=0.3]
(axis cs:1112000,0.662766491621998)
--(axis cs:1112000,0.655718553886382)
--(axis cs:375000,0.666133795976614)
--(axis cs:181000,0.695462937205578)
--(axis cs:164000,0.692200754422698)
--(axis cs:137000,0.702299497982434)
--(axis cs:93000,0.716863640780574)
--(axis cs:48000,0.741754101731269)
--(axis cs:20000,0.788673354952783)
--(axis cs:20000,0.80170837160819)
--(axis cs:20000,0.80170837160819)
--(axis cs:48000,0.753422262692319)
--(axis cs:93000,0.721032797503498)
--(axis cs:137000,0.715728111654488)
--(axis cs:164000,0.707536702605027)
--(axis cs:181000,0.705204693123578)
--(axis cs:375000,0.695923252361353)
--(axis cs:1112000,0.662766491621998)
--cycle;

\addplot [semithick, color0, mark=square*, mark size=3, mark options={solid}]
table {%
800000 0.644820269560838
721900 0.649308460606304
647500 0.64847134667128
576900 0.644266092607562
510000 0.645201593033889
447000 0.649726540518753
387700 0.657938047031214
332100 0.654505922361704
280300 0.655058392099674
232300 0.665977979976508
188100 0.663615195878902
147600 0.676370244253037
110900 0.675845852534807
78000 0.695327176854577
48800 0.716210916326926
23400 0.764818892035032
};
\addlegendentry{DenseNet}
\addplot [semithick, color1, mark=*, mark size=4, mark options={solid}]
table {%
1112000 0.65924252275419
375000 0.681028524168983
181000 0.700333815164578
164000 0.699868728513863
137000 0.709013804818461
93000 0.718948219142036
48000 0.747588182211794
20000 0.795190863280487
};
\addlegendentry{WideResNet}
\end{axis}

\end{tikzpicture}}
\end{minipage}
}%\hspace*{-0.3em}
\subfloat[AUPR$\uparrow$ on Cifar10]{
\begin{minipage}[b]{0.5\linewidth}
\label{fig:F_AUPR_C10}
\scalebox{0.45}{% This file was created by tikzplotlib v0.8.7.
\begin{tikzpicture}
\pgfplotsset{every axis/.append style={
                    label style={font=\large}
                    }}
\definecolor{color0}{rgb}{0.12156862745098,0.466666666666667,0.705882352941177}
\definecolor{color1}{rgb}{1,0.498039215686275,0.0549019607843137}

\begin{axis}[
log basis x={10},
tick align=inside,
%tick pos=both,
tick pos=left,
x grid style={white!69.01960784313725!black},
xlabel={Model Size (\# of parameters)},
xmin=14000, xmax=1750000,
xmode=log,
xtick style={color=black},
xtick={1000,10000,100000,1000000,10000000,100000000},
xticklabels={\(\displaystyle {10^{3}}\),\(\displaystyle {10^{4}}\),\(\displaystyle {10^{5}}\),\(\displaystyle {10^{6}}\),\(\displaystyle {10^{7}}\),\(\displaystyle {10^{8}}\)},
y grid style={white!69.01960784313725!black},
ylabel={AUPR},
ymin=0.408562731194153, ymax=0.569807020055292,
ytick style={color=black},
ytick={0.4,0.42,0.44,0.46,0.48,0.5,0.52,0.54,0.56,0.58},
yticklabels={0.40,0.42,0.44,0.46,0.48,0.50,0.52,0.54,0.56,0.58}
]
\path [fill=color0, fill opacity=0.3]
(axis cs:800000,0.465187415159157)
--(axis cs:800000,0.441193791600796)
--(axis cs:721900,0.415892017051478)
--(axis cs:647500,0.431319609233375)
--(axis cs:576900,0.454400362334466)
--(axis cs:510000,0.424406345073552)
--(axis cs:447000,0.446660133180795)
--(axis cs:387700,0.424173538153128)
--(axis cs:332100,0.442561916483806)
--(axis cs:280300,0.439838552271491)
--(axis cs:232300,0.436954280042097)
--(axis cs:188100,0.435693849783989)
--(axis cs:147600,0.446183244554587)
--(axis cs:110900,0.45106073275777)
--(axis cs:78000,0.467812002645899)
--(axis cs:48800,0.476678079499507)
--(axis cs:23400,0.505049056750171)
--(axis cs:23400,0.520647511075249)
--(axis cs:23400,0.520647511075249)
--(axis cs:48800,0.492404741619352)
--(axis cs:78000,0.49729461324639)
--(axis cs:110900,0.480125552731154)
--(axis cs:147600,0.489549723060956)
--(axis cs:188100,0.46465268190827)
--(axis cs:232300,0.471993294300061)
--(axis cs:280300,0.455677554915447)
--(axis cs:332100,0.456799787812045)
--(axis cs:387700,0.461295938247565)
--(axis cs:447000,0.482831613108413)
--(axis cs:510000,0.459742277685188)
--(axis cs:576900,0.464947952452494)
--(axis cs:647500,0.446123138465527)
--(axis cs:721900,0.430513296997986)
--(axis cs:800000,0.465187415159157)
--cycle;

\path [fill=color1, fill opacity=0.3]
(axis cs:1112000,0.445312843722866)
--(axis cs:1112000,0.423237489222752)
--(axis cs:375000,0.454381489236378)
--(axis cs:181000,0.476380726854636)
--(axis cs:164000,0.457707168357092)
--(axis cs:137000,0.479848548562571)
--(axis cs:93000,0.477655277609522)
--(axis cs:48000,0.482425934028783)
--(axis cs:20000,0.52180701526232)
--(axis cs:20000,0.562477734197968)
--(axis cs:20000,0.562477734197968)
--(axis cs:48000,0.524841958913509)
--(axis cs:93000,0.524824575710978)
--(axis cs:137000,0.497160770103356)
--(axis cs:164000,0.482810490675666)
--(axis cs:181000,0.487483637888488)
--(axis cs:375000,0.471857311146114)
--(axis cs:1112000,0.445312843722866)
--cycle;

\addplot [semithick, color0, mark=square*, mark size=3, mark options={solid}]
table {%
800000 0.453190603379977
721900 0.423202657024732
647500 0.438721373849451
576900 0.45967415739348
510000 0.44207431137937
447000 0.464745873144604
387700 0.442734738200346
332100 0.449680852147925
280300 0.447758053593469
232300 0.454473787171079
188100 0.450173265846129
147600 0.467866483807772
110900 0.465593142744462
78000 0.482553307946145
48800 0.48454141055943
23400 0.51284828391271
};
\addplot [semithick, color1, mark=*, mark size=4, mark options={solid}]
table {%
1112000 0.434275166472809
375000 0.463119400191246
181000 0.481932182371562
164000 0.470258829516379
137000 0.488504659332963
93000 0.50123992666025
48000 0.503633946471146
20000 0.542142374730144
};
\end{axis}

\end{tikzpicture}}
\end{minipage}
}\\%\hspace*{-0.3em}
\subfloat[AURC$\downarrow$ on Cifar100]{
\begin{minipage}[b]{0.5\linewidth}
\label{fig:F_AURC}
\scalebox{0.45}{% This file was created by tikzplotlib v0.8.7.
\begin{tikzpicture}
\pgfplotsset{every axis/.append style={
                    label style={font=\large}
                    }}
\definecolor{color0}{rgb}{0.12156862745098,0.466666666666667,0.705882352941177}
\definecolor{color1}{rgb}{1,0.498039215686275,0.0549019607843137}

\begin{axis}[
log basis x={10},
tick align=inside,
%tick pos=both,
tick pos=left,
x grid style={white!69.01960784313725!black},
xlabel={Model Size (\# of parameters)},
xmin=14000, xmax=1750000,
xmode=log,
xtick style={color=black},
xtick={1000,10000,100000,1000000,10000000,100000000},
xticklabels={\(\displaystyle {10^{3}}\),\(\displaystyle {10^{4}}\),\(\displaystyle {10^{5}}\),\(\displaystyle {10^{6}}\),\(\displaystyle {10^{7}}\),\(\displaystyle {10^{8}}\)},
y grid style={white!69.01960784313725!black},
ylabel={AURC},
ymin=0.0574308486827997, ymax=0.319231277003934,
ytick style={color=black},
ytick={0.05,0.1,0.15,0.2,0.25,0.3,0.35},
yticklabels={0.05,0.10,0.15,0.20,0.25,0.30,0.35}
]
\path [fill=color0, fill opacity=0.3]
(axis cs:800000,0.0726156203903154)
--(axis cs:800000,0.0693308681519422)
--(axis cs:721900,0.0716577900691246)
--(axis cs:647500,0.0735451733504976)
--(axis cs:576900,0.0767681669893746)
--(axis cs:510000,0.0781261807577962)
--(axis cs:447000,0.0797102606150418)
--(axis cs:387700,0.0833052376809248)
--(axis cs:332100,0.0888172305522539)
--(axis cs:280300,0.0939254129476911)
--(axis cs:232300,0.0998797600137231)
--(axis cs:188100,0.10665409660288)
--(axis cs:147600,0.113099631766896)
--(axis cs:110900,0.123038736918012)
--(axis cs:78000,0.134339249965882)
--(axis cs:48800,0.165527375161107)
--(axis cs:23400,0.269817485837882)
--(axis cs:23400,0.270882615204674)
--(axis cs:23400,0.270882615204674)
--(axis cs:48800,0.170452341247726)
--(axis cs:78000,0.136541532410875)
--(axis cs:110900,0.12600855030627)
--(axis cs:147600,0.116624508910066)
--(axis cs:188100,0.10892017001881)
--(axis cs:232300,0.103855330406935)
--(axis cs:280300,0.0958024779134216)
--(axis cs:332100,0.0921281976741489)
--(axis cs:387700,0.0862030816587315)
--(axis cs:447000,0.0844255484200892)
--(axis cs:510000,0.0799783676390613)
--(axis cs:576900,0.0777854603159441)
--(axis cs:647500,0.0769073223207773)
--(axis cs:721900,0.0728399470278558)
--(axis cs:800000,0.0726156203903154)
--cycle;

\path [fill=color1, fill opacity=0.3]
(axis cs:1112000,0.0827946147466608)
--(axis cs:1112000,0.0813549570445733)
--(axis cs:375000,0.107122210214346)
--(axis cs:181000,0.122291506438101)
--(axis cs:164000,0.122459665456)
--(axis cs:137000,0.131188359113005)
--(axis cs:93000,0.151955043071892)
--(axis cs:48000,0.19918607418266)
--(axis cs:20000,0.296719086107527)
--(axis cs:20000,0.307331257534792)
--(axis cs:20000,0.307331257534792)
--(axis cs:48000,0.204957022884683)
--(axis cs:93000,0.153633890908901)
--(axis cs:137000,0.135211881730574)
--(axis cs:164000,0.128559092423076)
--(axis cs:181000,0.124427259837526)
--(axis cs:375000,0.110988551088686)
--(axis cs:1112000,0.0827946147466608)
--cycle;

\addplot [semithick, color0, mark=square*, mark size=3, mark options={solid}]
table {%
800000 0.0709732442711288
721900 0.0722488685484902
647500 0.0752262478356374
576900 0.0772768136526594
510000 0.0790522741984288
447000 0.0820679045175655
387700 0.0847541596698282
332100 0.0904727141132014
280300 0.0948639454305563
232300 0.101867545210329
188100 0.107787133310845
147600 0.114862070338481
110900 0.124523643612141
78000 0.135440391188378
48800 0.167989858204417
23400 0.270350050521278
};
\addplot [semithick, color1, mark=*, mark size=4, mark options={solid}]
table {%
1112000 0.0820747858956171
375000 0.109055380651516
181000 0.123359383137813
164000 0.125509378939538
137000 0.13320012042179
93000 0.152794466990397
48000 0.202071548533672
20000 0.302025171821159
};
\end{axis}

\end{tikzpicture}}
\end{minipage}
}%\hspace*{-0.3em}
\subfloat[AURC$\downarrow$ on Cifar10]{
\begin{minipage}[b]{0.5\linewidth}
\label{fig:F_AURC_C10}
\scalebox{0.45}{% This file was created by tikzplotlib v0.8.7.
\begin{tikzpicture}
\pgfplotsset{every axis/.append style={
                    label style={font=\large}
                    }}
\definecolor{color0}{rgb}{0.12156862745098,0.466666666666667,0.705882352941177}
\definecolor{color1}{rgb}{1,0.498039215686275,0.0549019607843137}

\begin{axis}[
log basis x={10},
tick align=inside,
%tick pos=both,
tick pos=left,
x grid style={white!69.01960784313725!black},
xlabel={Model Size (\# of parameters)},
xmin=14000, xmax=1750000,
xmode=log,
xtick style={color=black},
xtick={1000,10000,100000,1000000,10000000,100000000},
xticklabels={\(\displaystyle {10^{3}}\),\(\displaystyle {10^{4}}\),\(\displaystyle {10^{5}}\),\(\displaystyle {10^{6}}\),\(\displaystyle {10^{7}}\),\(\displaystyle {10^{8}}\)},
y grid style={white!69.01960784313725!black},
ylabel={AURC},
ymin=0.00266165609756023, ymax=0.0474183076195497,
ytick style={color=black}
]
\path [fill=color0, fill opacity=0.3]
(axis cs:800000,0.00523194838114383)
--(axis cs:800000,0.00469604934855975)
--(axis cs:721900,0.00488232193625058)
--(axis cs:647500,0.00527592808296658)
--(axis cs:576900,0.00531042702635765)
--(axis cs:510000,0.00537982969153116)
--(axis cs:447000,0.00561682549600869)
--(axis cs:387700,0.00614673074545421)
--(axis cs:332100,0.00668904423436316)
--(axis cs:280300,0.00659296155105351)
--(axis cs:232300,0.00811745549245328)
--(axis cs:188100,0.00851457922463002)
--(axis cs:147600,0.00980155777715518)
--(axis cs:110900,0.0118267595292171)
--(axis cs:78000,0.0144088403043476)
--(axis cs:48800,0.020843938038873)
--(axis cs:23400,0.0389533490367722)
--(axis cs:23400,0.0405049613140278)
--(axis cs:23400,0.0405049613140278)
--(axis cs:48800,0.0222709184448461)
--(axis cs:78000,0.0161397039708498)
--(axis cs:110900,0.0123387535799098)
--(axis cs:147600,0.00997909831098241)
--(axis cs:188100,0.00886343550982689)
--(axis cs:232300,0.00840279888720376)
--(axis cs:280300,0.00699566718653261)
--(axis cs:332100,0.00675582488765889)
--(axis cs:387700,0.00683524918643812)
--(axis cs:447000,0.00604199462703265)
--(axis cs:510000,0.0056694786463073)
--(axis cs:576900,0.00573952839198395)
--(axis cs:647500,0.00573779824811005)
--(axis cs:721900,0.00524827817725457)
--(axis cs:800000,0.00523194838114383)
--cycle;

\path [fill=color1, fill opacity=0.3]
(axis cs:1112000,0.00845362309464048)
--(axis cs:1112000,0.00675922171972274)
--(axis cs:375000,0.00937144647156793)
--(axis cs:181000,0.0120259377745725)
--(axis cs:164000,0.0122659877741912)
--(axis cs:137000,0.0130270526450292)
--(axis cs:93000,0.0164313328290597)
--(axis cs:48000,0.025028445563541)
--(axis cs:20000,0.0434409065260315)
--(axis cs:20000,0.0453839143685501)
--(axis cs:20000,0.0453839143685501)
--(axis cs:48000,0.0253406235892378)
--(axis cs:93000,0.0169917224897387)
--(axis cs:137000,0.0138173794921974)
--(axis cs:164000,0.0127108337126913)
--(axis cs:181000,0.0132795412581774)
--(axis cs:375000,0.00989324149952833)
--(axis cs:1112000,0.00845362309464048)
--cycle;

\addplot [semithick, color0, mark=square*, mark size=3, mark options={solid}]
table {%
800000 0.00496399886485179
721900 0.00506530005675257
647500 0.00550686316553831
576900 0.0055249777091708
510000 0.00552465416891923
447000 0.00582941006152067
387700 0.00649098996594617
332100 0.00672243456101102
280300 0.00679431436879306
232300 0.00826012718982852
188100 0.00868900736722845
147600 0.0098903280440688
110900 0.0120827565545634
78000 0.0152742721375987
48800 0.0215574282418596
23400 0.0397291551754
};
\addplot [semithick, color1, mark=*, mark size=4, mark options={solid}]
table {%
1112000 0.00760642240718161
375000 0.00963234398554813
181000 0.012652739516375
164000 0.0124884107434413
137000 0.0134222160686133
93000 0.0167115276593992
48000 0.0251845345763894
20000 0.0444124104472908
};
\end{axis}

\end{tikzpicture}}
\end{minipage}
}
\caption{Effect of model complexity on selective prediction.}
\label{fig:CurveCompareSP_C10}
\vspace{-0.04in}
%\vspace{-0.05in}
\end{figure}

\textbf{The performance of confidence calibration is insensitive to the model size.} For confidence calibration, the uncertainty estimation quality measured by AECE and AMCE is shown in Figure~\ref{fig:CurveCompareECEMCE}. We also plot the results measured by ECE and MCE in the same figures to validate the discussion. We find that the estimation quality remains almost flat and does not show a strong trend with the model complexity in terms of AECE and AMCE. This is a mixed result of a number of factors including model accuracy, the effectiveness of temperature scaling, and the confidence distribution. 

In terms of the effect of adaptive binning, it is observed that AECE is generally bigger than ECE by a small margin. The reason is that different binning methods lead to different levels of internal compensation as discussed above. The adaptive binning used by AECE creates 12.6 bins and 21.2 bins on average for Cifar10 and Cifar100 respectively, which are more than the 10 equal-range bins used in ECE~\cite{guo2017calibration,kumar2018trainable,sander2019towards}.  Note that Cifar100 gets more bins than Cifar10 because Cifar100 is more difficult and the confidence distribution is significantly flatter, which naturally enables more bins with accurate accuracy estimation. This further validates the superiority of the adaptive binning.  Note that AECE is very close to ECE even when the underline bins are very different, because most of the samples are in the bin with the highest confidence (the case is more severe when the model has high accuracy) and these samples dominate the value of ECE and AECE. The advantage of using adaptive binning is much more significant in reliability diagrams and AMCE compared with the ``expected'' calibration error.

Meanwhile, as shown in Figure~\ref{fig:F_AMCE}, MCE is close to AMCE for WideResNet but significantly bigger than AMCE in some cases for DenseNet. The reason is that DenseNet tends to have confidence distributions that are more concentrated in the high confidence area. As a result, the inaccurate accuracy estimation in the bins with a small number of samples is exposed, and this leads to some undesired big calibration error. This is further validated in the results on Cifar10 where both WideResNet and DenseNet have more non-uniform confidence distributions because of the easier task. As shown in Figure~\ref{fig:F_AMCE_C10}, the MCE for both WideResNet and DenseNet are unstable and significantly higher than AMCE indicating an even worse situation caused by the inaccurate accuracy estimation. In summary, the comparison between the baseline and adaptive binning validates our discussion and design intuition for adaptive binning.

\begin{figure}[htb]
\vspace{-0.05in}
\centering
%\hspace*{-2.2em}
\subfloat[AECE$\downarrow$ on Cifar100]{
\begin{minipage}[b]{0.5\linewidth}
\label{fig:F_AECE}
\scalebox{0.45}{% This file was created by tikzplotlib v0.8.7.
\begin{tikzpicture}
\pgfplotsset{every axis/.append style={
                    label style={font=\large}
                    }}
\definecolor{color1}{rgb}{1,0.498039215686275,0.0549019607843137}
\definecolor{color3}{rgb}{0.83921568627451,0.152941176470588,0.156862745098039}
\definecolor{color2}{rgb}{0.172549019607843,0.627450980392157,0.172549019607843}
\definecolor{color0}{rgb}{0.12156862745098,0.466666666666667,0.705882352941177}

\begin{axis}[
legend cell align={left},
legend style={fill opacity=0.8, draw opacity=1, text opacity=1, at={(0.5,0.91)}, anchor=north, draw=white!80.0!black},
log basis x={10},
tick align=inside,
%tick pos=both,
tick pos=left,
x grid style={lightgray!92.02614379084967!black},
xlabel={Model Size (\# of parameters)},
xmin=14000, xmax=1750000,
xmode=log,
xtick style={color=black},
xtick={1000,10000,100000,1000000,10000000,100000000},
xticklabels={\(\displaystyle {10^{3}}\),\(\displaystyle {10^{4}}\),\(\displaystyle {10^{5}}\),\(\displaystyle {10^{6}}\),\(\displaystyle {10^{7}}\),\(\displaystyle {10^{8}}\)},
y grid style={lightgray!92.02614379084967!black},
ylabel={ECE},
ymin=0.00409155673491346, ymax=0.0287180508924395,
ytick style={color=black}
]
\path [fill=color0, fill opacity=0.3]
(axis cs:800000,0.0236995580887847)
--(axis cs:800000,0.0186542798804218)
--(axis cs:721900,0.0179878855489249)
--(axis cs:647500,0.0205494907044303)
--(axis cs:576900,0.020400214714018)
--(axis cs:510000,0.0151623517313721)
--(axis cs:447000,0.0158467155149736)
--(axis cs:387700,0.0117817239973264)
--(axis cs:332100,0.0182381160215034)
--(axis cs:280300,0.0135045746166744)
--(axis cs:232300,0.0115856525590448)
--(axis cs:188100,0.0121222578744799)
--(axis cs:147600,0.0127291554585566)
--(axis cs:110900,0.0164541132253411)
--(axis cs:78000,0.014652039598769)
--(axis cs:48800,0.0136249379348804)
--(axis cs:23400,0.0179805864608808)
--(axis cs:23400,0.0263027000080253)
--(axis cs:23400,0.0263027000080253)
--(axis cs:48800,0.0167007067353765)
--(axis cs:78000,0.0175035801768541)
--(axis cs:110900,0.0214694387338027)
--(axis cs:147600,0.0172031671276177)
--(axis cs:188100,0.0156030553887735)
--(axis cs:232300,0.0141893758792496)
--(axis cs:280300,0.0143397645961999)
--(axis cs:332100,0.0228083828075021)
--(axis cs:387700,0.0154136233898189)
--(axis cs:447000,0.019279765136004)
--(axis cs:510000,0.0178663745245492)
--(axis cs:576900,0.0216273674623053)
--(axis cs:647500,0.022418480763797)
--(axis cs:721900,0.0200140482096529)
--(axis cs:800000,0.0236995580887847)
--cycle;

\path [fill=color1, fill opacity=0.3]
(axis cs:1112000,0.0275986647943702)
--(axis cs:1112000,0.0243423227642072)
--(axis cs:375000,0.0133551014377911)
--(axis cs:181000,0.0150495873688302)
--(axis cs:164000,0.0118658615562256)
--(axis cs:137000,0.0170364642494844)
--(axis cs:93000,0.016727589002146)
--(axis cs:48000,0.0155444335344967)
--(axis cs:20000,0.017110621588254)
--(axis cs:20000,0.0218669668570882)
--(axis cs:20000,0.0218669668570882)
--(axis cs:48000,0.0182493294052739)
--(axis cs:93000,0.0233549673784862)
--(axis cs:137000,0.0186995599636622)
--(axis cs:164000,0.0198436905670635)
--(axis cs:181000,0.0198195621010052)
--(axis cs:375000,0.0154392412661834)
--(axis cs:1112000,0.0275986647943702)
--cycle;

\path [fill=color2, fill opacity=0.3]
(axis cs:800000,0.0205541262912465)
--(axis cs:800000,0.0155731674007839)
--(axis cs:721900,0.0161438429204543)
--(axis cs:647500,0.0127347160268721)
--(axis cs:576900,0.0171294097511722)
--(axis cs:510000,0.0121598041881438)
--(axis cs:447000,0.0146654826218185)
--(axis cs:387700,0.00774306929178629)
--(axis cs:332100,0.01426096422747)
--(axis cs:280300,0.00783327661108336)
--(axis cs:232300,0.00949642439382907)
--(axis cs:188100,0.00521094283298283)
--(axis cs:147600,0.00821527210305887)
--(axis cs:110900,0.0108484532065025)
--(axis cs:78000,0.00805200535722836)
--(axis cs:48800,0.0101848483779373)
--(axis cs:23400,0.0174775997690318)
--(axis cs:23400,0.0228142574107497)
--(axis cs:23400,0.0228142574107497)
--(axis cs:48800,0.013527976453458)
--(axis cs:78000,0.0127413401475465)
--(axis cs:110900,0.019782954742939)
--(axis cs:147600,0.0134001037062751)
--(axis cs:188100,0.0100113738689689)
--(axis cs:232300,0.0108636832741872)
--(axis cs:280300,0.0101202458958461)
--(axis cs:332100,0.0198070734063847)
--(axis cs:387700,0.0132046085308527)
--(axis cs:447000,0.0168795233478045)
--(axis cs:510000,0.0159350165131006)
--(axis cs:576900,0.0199777555952429)
--(axis cs:647500,0.016002372569525)
--(axis cs:721900,0.0169473769648848)
--(axis cs:800000,0.0205541262912465)
--cycle;

\path [fill=color3, fill opacity=0.3]
(axis cs:1112000,0.0269128592980198)
--(axis cs:1112000,0.024292137521788)
--(axis cs:375000,0.00763731299333163)
--(axis cs:181000,0.0107183299554013)
--(axis cs:164000,0.00908414562130107)
--(axis cs:137000,0.0126129746922991)
--(axis cs:93000,0.013258274950512)
--(axis cs:48000,0.0104142176232396)
--(axis cs:20000,0.0132299436155767)
--(axis cs:20000,0.0176407945762155)
--(axis cs:20000,0.0176407945762155)
--(axis cs:48000,0.0137469656820982)
--(axis cs:93000,0.018681128019439)
--(axis cs:137000,0.0160432680367219)
--(axis cs:164000,0.0152708405308239)
--(axis cs:181000,0.0169790883888633)
--(axis cs:375000,0.0107881420315529)
--(axis cs:1112000,0.0269128592980198)
--cycle;

\addplot [semithick, color0, mark=square*, mark size=3, mark options={solid}]
table {%
800000 0.0211769189846033
721900 0.0190009668792889
647500 0.0214839857341136
576900 0.0210137910881616
510000 0.0165143631279606
447000 0.0175632403254888
387700 0.0135976736935726
332100 0.0205232494145028
280300 0.0139221696064372
232300 0.0128875142191472
188100 0.0138626566316267
147600 0.0149661612930871
110900 0.0189617759795719
78000 0.0160778098878115
48800 0.0151628223351284
23400 0.0221416432344531
};
\addlegendentry{DenseNet AECE}
\addplot [semithick, color1, mark=*, mark size=4, mark options={solid}]
table {%
1112000 0.0259704937792887
375000 0.0143971713519872
181000 0.0174345747349177
164000 0.0158547760616446
137000 0.0178680121065733
93000 0.0200412781903161
48000 0.0168968814698853
20000 0.0194887942226711
};
\addlegendentry{WideResNet AECE}
\addplot [semithick, color2, mark=+, mark size=4, mark options={solid}]
table {%
800000 0.0180636468460152
721900 0.0165456099426696
647500 0.0143685442981986
576900 0.0185535826732076
510000 0.0140474103506222
447000 0.0157725029848115
387700 0.0104738389113195
332100 0.0170340188169274
280300 0.00897676125346471
232300 0.0101800538340081
188100 0.00761115835097585
147600 0.010807687904667
110900 0.0153157039747207
78000 0.0103966727523874
48800 0.0118564124156976
23400 0.0201459285898908
};
\addlegendentry{DenseNet ECE}
\addplot [semithick, color3, mark=x, mark size=4, mark options={solid}]
table {%
1112000 0.0256024984099039
375000 0.00921272751244227
181000 0.0138487091721323
164000 0.0121774930760625
137000 0.0143281213645105
93000 0.0159697014849755
48000 0.0120805916526689
20000 0.0154353690958961
};
\addlegendentry{WideResNet ECE}
\end{axis}

\end{tikzpicture}}
\end{minipage}
}%\hspace*{-0.3em}
\subfloat[AECE$\downarrow$ on Cifar10]{
\begin{minipage}[b]{0.5\linewidth}
\label{fig:F_AECE_C10}
\scalebox{0.45}{% This file was created by tikzplotlib v0.8.7.
\begin{tikzpicture}
\pgfplotsset{every axis/.append style={
                    label style={font=\large}
                    }}
\definecolor{color1}{rgb}{1,0.498039215686275,0.0549019607843137}
\definecolor{color3}{rgb}{0.83921568627451,0.152941176470588,0.156862745098039}
\definecolor{color2}{rgb}{0.172549019607843,0.627450980392157,0.172549019607843}
\definecolor{color0}{rgb}{0.12156862745098,0.466666666666667,0.705882352941177}

\begin{axis}[
legend cell align={left},
legend style={fill opacity=0.8, draw opacity=1, text opacity=1, draw=white!80.0!black},
log basis x={10},
tick align=inside,
%tick pos=both,
tick pos=left,
x grid style={lightgray!92.02614379084967!black},
xlabel={Model Size (\# of parameters)},
xmin=14000, xmax=1750000,
xmode=log,
xtick style={color=black},
xtick={1000,10000,100000,1000000,10000000,100000000},
xticklabels={\(\displaystyle {10^{3}}\),\(\displaystyle {10^{4}}\),\(\displaystyle {10^{5}}\),\(\displaystyle {10^{6}}\),\(\displaystyle {10^{7}}\),\(\displaystyle {10^{8}}\)},
y grid style={lightgray!92.02614379084967!black},
ylabel={ECE},
ymin=0.00195414243969605, ymax=0.0212602450394212,
ytick style={color=black}
]
\path [fill=color0, fill opacity=0.3]
(axis cs:800000,0.0057567389346657)
--(axis cs:800000,0.00464070622531513)
--(axis cs:721900,0.00573193286827544)
--(axis cs:647500,0.00388304460183998)
--(axis cs:576900,0.00612234538136652)
--(axis cs:510000,0.00607715276651511)
--(axis cs:447000,0.00551603488294407)
--(axis cs:387700,0.00710796689311172)
--(axis cs:332100,0.00859563455834774)
--(axis cs:280300,0.00517136278275326)
--(axis cs:232300,0.0106847260607946)
--(axis cs:188100,0.00653158509667579)
--(axis cs:147600,0.00667300690515136)
--(axis cs:110900,0.00771324635605421)
--(axis cs:78000,0.00609936520026678)
--(axis cs:48800,0.0082207979626731)
--(axis cs:23400,0.0110830590787443)
--(axis cs:23400,0.0117023683472307)
--(axis cs:23400,0.0117023683472307)
--(axis cs:48800,0.00910125520842886)
--(axis cs:78000,0.0105445746132771)
--(axis cs:110900,0.00884836388793591)
--(axis cs:147600,0.00859439176505914)
--(axis cs:188100,0.00789720843958632)
--(axis cs:232300,0.0113623590203536)
--(axis cs:280300,0.00581914361307604)
--(axis cs:332100,0.0125055869869481)
--(axis cs:387700,0.00900505879051256)
--(axis cs:447000,0.00672277708357708)
--(axis cs:510000,0.00737607292786758)
--(axis cs:576900,0.00852868571493996)
--(axis cs:647500,0.00554381443280123)
--(axis cs:721900,0.00733770017332554)
--(axis cs:800000,0.0057567389346657)
--cycle;

\path [fill=color1, fill opacity=0.3]
(axis cs:1112000,0.0108017791306085)
--(axis cs:1112000,0.0095133414994811)
--(axis cs:375000,0.00845621385215551)
--(axis cs:181000,0.00800189716831559)
--(axis cs:164000,0.00792483338373466)
--(axis cs:137000,0.00720318069293671)
--(axis cs:93000,0.00746518904240154)
--(axis cs:48000,0.00878597074990143)
--(axis cs:20000,0.0151310477064666)
--(axis cs:20000,0.0203826949212519)
--(axis cs:20000,0.0203826949212519)
--(axis cs:48000,0.0113525001202537)
--(axis cs:93000,0.0101910818644414)
--(axis cs:137000,0.00763456344946215)
--(axis cs:164000,0.00997240919109818)
--(axis cs:181000,0.00970328542267326)
--(axis cs:375000,0.0105342381736511)
--(axis cs:1112000,0.0108017791306085)
--cycle;

\path [fill=color2, fill opacity=0.3]
(axis cs:800000,0.00520153712946876)
--(axis cs:800000,0.00441634115857958)
--(axis cs:721900,0.0045696044115508)
--(axis cs:647500,0.00283169255786537)
--(axis cs:576900,0.0034976870900085)
--(axis cs:510000,0.00490215347655923)
--(axis cs:447000,0.00337621061063684)
--(axis cs:387700,0.00609191441470854)
--(axis cs:332100,0.00852296504588519)
--(axis cs:280300,0.00397737839300943)
--(axis cs:232300,0.00996530546754078)
--(axis cs:188100,0.00648872830704686)
--(axis cs:147600,0.0057556036013836)
--(axis cs:110900,0.00702944173032099)
--(axis cs:78000,0.0055625253139268)
--(axis cs:48800,0.0066884697032627)
--(axis cs:23400,0.00621188389708096)
--(axis cs:23400,0.00934019552054381)
--(axis cs:23400,0.00934019552054381)
--(axis cs:48800,0.00962786095971215)
--(axis cs:78000,0.00862696017880027)
--(axis cs:110900,0.00830829227583182)
--(axis cs:147600,0.00836817342508955)
--(axis cs:188100,0.00746636067454324)
--(axis cs:232300,0.0114403617571984)
--(axis cs:280300,0.00412874508619202)
--(axis cs:332100,0.0120271949143096)
--(axis cs:387700,0.00724874623522398)
--(axis cs:447000,0.00455005823997207)
--(axis cs:510000,0.00865319615785698)
--(axis cs:576900,0.00639382256765166)
--(axis cs:647500,0.00359118289446645)
--(axis cs:721900,0.00653819180695186)
--(axis cs:800000,0.00520153712946876)
--cycle;

\path [fill=color3, fill opacity=0.3]
(axis cs:1112000,0.00898990810121742)
--(axis cs:1112000,0.00807253738038392)
--(axis cs:375000,0.00722247079169642)
--(axis cs:181000,0.0072419228810968)
--(axis cs:164000,0.00539573107988911)
--(axis cs:137000,0.00525628123526369)
--(axis cs:93000,0.00766762716480337)
--(axis cs:48000,0.0067497560328096)
--(axis cs:20000,0.0137729234278348)
--(axis cs:20000,0.0191799724371527)
--(axis cs:20000,0.0191799724371527)
--(axis cs:48000,0.0105507223677645)
--(axis cs:93000,0.0108458346145404)
--(axis cs:137000,0.00548213498643848)
--(axis cs:164000,0.00780267104564113)
--(axis cs:181000,0.0077742457815496)
--(axis cs:375000,0.00958274076193644)
--(axis cs:1112000,0.00898990810121742)
--cycle;

\addplot [semithick, color0, mark=square*, mark size=3, mark options={solid}]
table {%
800000 0.00519872257999042
721900 0.00653481652080049
647500 0.0047134295173206
576900 0.00732551554815324
510000 0.00672661284719134
447000 0.00611940598326058
387700 0.00805651284181214
332100 0.0105506107726479
280300 0.00549525319791465
232300 0.0110235425405741
188100 0.00721439676813106
147600 0.00763369933510525
110900 0.00828080512199506
78000 0.00832196990677195
48800 0.00866102658555098
23400 0.0113927137129875
};
\addlegendentry{DenseNet AECE}
\addplot [semithick, color1, mark=*, mark size=3, mark options={solid}]
table {%
1112000 0.0101575603150448
375000 0.00949522601290332
181000 0.00885259129549443
164000 0.00894862128741642
137000 0.00741887207119943
93000 0.00882813545342148
48000 0.0100692354350776
20000 0.0177568713138592
};
\addlegendentry{WideResNet AECE}
\addplot [semithick, color2, mark=+, mark size=4, mark options={solid}]
table {%
800000 0.00480893914402417
721900 0.00555389810925133
647500 0.00321143772616591
576900 0.00494575482883008
510000 0.0067776748172081
447000 0.00396313442530446
387700 0.00667033032496626
332100 0.0102750799800974
280300 0.00405306173960073
232300 0.0107028336123696
188100 0.00697754449079505
147600 0.00706188851323657
110900 0.0076688670030764
78000 0.00709474274636353
48800 0.00815816533148742
23400 0.00777603970881239
};
\addlegendentry{DenseNet ECE}
\addplot [semithick, color3, mark=x, mark size=4, mark options={solid}]
table {%
1112000 0.00853122274080067
375000 0.00840260577681643
181000 0.0075080843313232
164000 0.00659920106276512
137000 0.00536920811085109
93000 0.0092567308896719
48000 0.00865023920028704
20000 0.0164764479324938
};
\addlegendentry{WideResNet ECE}
\end{axis}

\end{tikzpicture}}
\end{minipage}
}\\%\hspace*{-0.3em}
%\vspace{-0.02in}
\subfloat[AMCE$\downarrow$ on Cifar100]{
\begin{minipage}[b]{0.5\linewidth}
\label{fig:F_AMCE}
\scalebox{0.45}{% This file was created by tikzplotlib v0.8.7.
\begin{tikzpicture}
\pgfplotsset{every axis/.append style={
                    label style={font=\large}
                    }}
\definecolor{color1}{rgb}{1,0.498039215686275,0.0549019607843137}
\definecolor{color3}{rgb}{0.83921568627451,0.152941176470588,0.156862745098039}
\definecolor{color2}{rgb}{0.172549019607843,0.627450980392157,0.172549019607843}
\definecolor{color0}{rgb}{0.12156862745098,0.466666666666667,0.705882352941177}

\begin{axis}[
legend cell align={left},
legend style={fill opacity=0.8, draw opacity=1, text opacity=1, draw=white!80.0!black},
log basis x={10},
tick align=inside,
%tick pos=both,
tick pos=left,
x grid style={lightgray!92.02614379084967!black},
xlabel={Model Size (\# of parameters)},
xmin=14000, xmax=1750000,
xmode=log,
xtick style={color=black},
xtick={1000,10000,100000,1000000,10000000,100000000},
xticklabels={\(\displaystyle {10^{3}}\),\(\displaystyle {10^{4}}\),\(\displaystyle {10^{5}}\),\(\displaystyle {10^{6}}\),\(\displaystyle {10^{7}}\),\(\displaystyle {10^{8}}\)},
y grid style={lightgray!92.02614379084967!black},
ylabel={MCE},
ymin=0.0041169848318501, ymax=0.369763693114623,
ytick style={color=black},
ytick={0,0.05,0.1,0.15,0.2,0.25,0.3,0.35,0.4},
yticklabels={0.00,0.05,0.10,0.15,0.20,0.25,0.30,0.35,0.40}
]
\path [fill=color0, fill opacity=0.3]
(axis cs:800000,0.0897615500411701)
--(axis cs:800000,0.0798467647874771)
--(axis cs:721900,0.0443587843084724)
--(axis cs:647500,0.0576752688272116)
--(axis cs:576900,0.0616757028923369)
--(axis cs:510000,0.061726008993178)
--(axis cs:447000,0.0634096263125965)
--(axis cs:387700,0.0414358312463613)
--(axis cs:332100,0.0593113391257016)
--(axis cs:280300,0.0595848938449928)
--(axis cs:232300,0.0362994113832916)
--(axis cs:188100,0.0475437068927824)
--(axis cs:147600,0.0347592763233453)
--(axis cs:110900,0.0521514756742755)
--(axis cs:78000,0.0403528958616479)
--(axis cs:48800,0.0423907382621525)
--(axis cs:23400,0.0495550541405466)
--(axis cs:23400,0.0704023808697282)
--(axis cs:23400,0.0704023808697282)
--(axis cs:48800,0.048147243733463)
--(axis cs:78000,0.0540972541549357)
--(axis cs:110900,0.0833970487919758)
--(axis cs:147600,0.0514000978162495)
--(axis cs:188100,0.0629366475416584)
--(axis cs:232300,0.060604300365337)
--(axis cs:280300,0.081884499577425)
--(axis cs:332100,0.0734373674855171)
--(axis cs:387700,0.0742629526501192)
--(axis cs:447000,0.0766863463853111)
--(axis cs:510000,0.0919383041512666)
--(axis cs:576900,0.126456905978039)
--(axis cs:647500,0.069863206107467)
--(axis cs:721900,0.0607777838094879)
--(axis cs:800000,0.0897615500411701)
--cycle;

\path [fill=color1, fill opacity=0.3]
(axis cs:1112000,0.0935206880257047)
--(axis cs:1112000,0.075329108738295)
--(axis cs:375000,0.0469351285646064)
--(axis cs:181000,0.0561347822289516)
--(axis cs:164000,0.0369665985313453)
--(axis cs:137000,0.0551320244486021)
--(axis cs:93000,0.0446398059477276)
--(axis cs:48000,0.0511052683936145)
--(axis cs:20000,0.0484438796485265)
--(axis cs:20000,0.056286805168589)
--(axis cs:20000,0.056286805168589)
--(axis cs:48000,0.0601042746382873)
--(axis cs:93000,0.0603874523920439)
--(axis cs:137000,0.0798052735130471)
--(axis cs:164000,0.0494844373227329)
--(axis cs:181000,0.0640889489424829)
--(axis cs:375000,0.0698165092137405)
--(axis cs:1112000,0.0935206880257047)
--cycle;

\path [fill=color2, fill opacity=0.3]
(axis cs:800000,0.109051206101533)
--(axis cs:800000,0.0751010747424114)
--(axis cs:721900,0.0590849459017157)
--(axis cs:647500,0.0582481477991595)
--(axis cs:576900,0.0824574911034416)
--(axis cs:510000,0.0931434051914539)
--(axis cs:447000,0.0577100927660447)
--(axis cs:387700,0.0551317280837618)
--(axis cs:332100,0.0551269267097487)
--(axis cs:280300,0.060334167409702)
--(axis cs:232300,0.088054914059954)
--(axis cs:188100,0.0797720231716074)
--(axis cs:147600,0.0534603271332078)
--(axis cs:110900,0.0634401380219287)
--(axis cs:78000,0.086251131127038)
--(axis cs:48800,0.035512515588348)
--(axis cs:23400,0.0429960762426173)
--(axis cs:23400,0.0483236618598296)
--(axis cs:23400,0.0483236618598296)
--(axis cs:48800,0.116547709109401)
--(axis cs:78000,0.0935908943022215)
--(axis cs:110900,0.353143388192679)
--(axis cs:147600,0.349553163045774)
--(axis cs:188100,0.122964178264777)
--(axis cs:232300,0.106266331658092)
--(axis cs:280300,0.115436363200888)
--(axis cs:332100,0.148963591800249)
--(axis cs:387700,0.147438549582135)
--(axis cs:447000,0.0945587853198209)
--(axis cs:510000,0.147645273536216)
--(axis cs:576900,0.270728846104954)
--(axis cs:647500,0.0868409626377204)
--(axis cs:721900,0.0864791215907229)
--(axis cs:800000,0.109051206101533)
--cycle;

\path [fill=color3, fill opacity=0.3]
(axis cs:1112000,0.0749177721200221)
--(axis cs:1112000,0.0595632157166672)
--(axis cs:375000,0.043073986058284)
--(axis cs:181000,0.0271568955185406)
--(axis cs:164000,0.0207372897537943)
--(axis cs:137000,0.0336361575396295)
--(axis cs:93000,0.0353458984726293)
--(axis cs:48000,0.0405265880307265)
--(axis cs:20000,0.0299043472161706)
--(axis cs:20000,0.0341548789460758)
--(axis cs:20000,0.0341548789460758)
--(axis cs:48000,0.0534394921444444)
--(axis cs:93000,0.0791567625734664)
--(axis cs:137000,0.046469617983916)
--(axis cs:164000,0.0909797038974325)
--(axis cs:181000,0.0459021023273872)
--(axis cs:375000,0.107362243546238)
--(axis cs:1112000,0.0749177721200221)
--cycle;

\addplot [semithick, color0, mark=square*, mark size=3, mark options={solid}]
table {%
800000 0.0848041574143236
721900 0.0525682840589802
647500 0.0637692374673393
576900 0.0940663044351877
510000 0.0768321565722223
447000 0.0700479863489538
387700 0.0578493919482402
332100 0.0663743533056093
280300 0.0707346967112089
232300 0.0484518558743143
188100 0.0552401772172204
147600 0.0430796870697974
110900 0.0677742622331256
78000 0.0472250750082918
48800 0.0452689909978077
23400 0.0599787175051374
};
\addlegendentry{DenseNet AMCE}
\addplot [semithick, color1, mark=*, mark size=4, mark options={solid}]
table {%
1112000 0.0844248983819998
375000 0.0583758188891735
181000 0.0601118655857172
164000 0.0432255179270391
137000 0.0674686489808246
93000 0.0525136291698857
48000 0.0556047715159509
20000 0.0523653424085578
};
\addlegendentry{WideResNet AMCE}
\addplot [semithick, color2, mark=+, mark size=4, mark options={solid}]
table {%
800000 0.0920761404219723
721900 0.0727820337462193
647500 0.07254455521844
576900 0.176593168604198
510000 0.120394339363835
447000 0.0761344390429328
387700 0.101285138832948
332100 0.102045259254999
280300 0.0878852653052948
232300 0.0971606228590231
188100 0.101368100718192
147600 0.201506745089491
110900 0.208291763107304
78000 0.0899210127146298
48800 0.0760301123488746
23400 0.0456598690512234
};
\addlegendentry{DenseNet MCE}
\addplot [semithick, color3, mark=x, mark size=4, mark options={solid}]
table {%
1112000 0.0672404939183447
375000 0.075218114802261
181000 0.0365294989229639
164000 0.0558584968256134
137000 0.0400528877617727
93000 0.0572513305230479
48000 0.0469830400875854
20000 0.0320296130811232
};
\addlegendentry{WideResNet MCE}
\end{axis}

\end{tikzpicture}}
\end{minipage}
}%\hspace*{-0.3em}
\subfloat[AMCE$\downarrow$ on Cifar10]{
\begin{minipage}[b]{0.5\linewidth}
\label{fig:F_AMCE_C10}
\scalebox{0.45}{% This file was created by tikzplotlib v0.8.7.
\begin{tikzpicture}
\pgfplotsset{every axis/.append style={
                    label style={font=\large}
                    }}
\definecolor{color1}{rgb}{1,0.498039215686275,0.0549019607843137}
\definecolor{color3}{rgb}{0.83921568627451,0.152941176470588,0.156862745098039}
\definecolor{color2}{rgb}{0.172549019607843,0.627450980392157,0.172549019607843}
\definecolor{color0}{rgb}{0.12156862745098,0.466666666666667,0.705882352941177}

\begin{axis}[
legend cell align={left},
legend style={fill opacity=0.8, draw opacity=1, text opacity=1, draw=white!80.0!black},
log basis x={10},
tick align=inside,
%tick pos=both,
tick pos=left,
x grid style={lightgray!92.02614379084967!black},
xlabel={Model Size (\# of parameters)},
xmin=14000, xmax=1750000,
xmode=log,
xtick style={color=black},
xtick={1000,10000,100000,1000000,10000000,100000000},
xticklabels={\(\displaystyle {10^{3}}\),\(\displaystyle {10^{4}}\),\(\displaystyle {10^{5}}\),\(\displaystyle {10^{6}}\),\(\displaystyle {10^{7}}\),\(\displaystyle {10^{8}}\)},
y grid style={lightgray!92.02614379084967!black},
ylabel={MCE},
ymin=-0.0645614713315986, ymax=0.946507646091831,
ytick style={color=black},
ytick={-0.2,0,0.2,0.4,0.6,0.8,1},
yticklabels={−0.2,0.0,0.2,0.4,0.6,0.8,1.0}
]
\path [fill=color0, fill opacity=0.3]
(axis cs:800000,0.0704012410634373)
--(axis cs:800000,0.0380395416835758)
--(axis cs:721900,0.0561711580524805)
--(axis cs:647500,0.0642784892060841)
--(axis cs:576900,0.0603280114226516)
--(axis cs:510000,0.0431141773418893)
--(axis cs:447000,0.046148771220321)
--(axis cs:387700,0.064004335160848)
--(axis cs:332100,0.0643726202599479)
--(axis cs:280300,0.034595382093348)
--(axis cs:232300,0.0678622917461629)
--(axis cs:188100,0.0536310276825056)
--(axis cs:147600,0.0415148382602637)
--(axis cs:110900,0.0648693695441605)
--(axis cs:78000,0.0618437234580254)
--(axis cs:48800,0.0608909455548582)
--(axis cs:23400,0.0604950508228535)
--(axis cs:23400,0.0741757252257403)
--(axis cs:23400,0.0741757252257403)
--(axis cs:48800,0.0984764495453868)
--(axis cs:78000,0.0988456014060204)
--(axis cs:110900,0.0900333803495642)
--(axis cs:147600,0.0605366377724407)
--(axis cs:188100,0.0914946176315805)
--(axis cs:232300,0.12479826969204)
--(axis cs:280300,0.0722329705810046)
--(axis cs:332100,0.0911495678366274)
--(axis cs:387700,0.117755617346238)
--(axis cs:447000,0.0663854922732471)
--(axis cs:510000,0.103613174458203)
--(axis cs:576900,0.0890541564160412)
--(axis cs:647500,0.116177105740107)
--(axis cs:721900,0.0812612520473186)
--(axis cs:800000,0.0704012410634373)
--cycle;

\path [fill=color1, fill opacity=0.3]
(axis cs:1112000,0.089941400454291)
--(axis cs:1112000,0.0663227964808471)
--(axis cs:375000,0.0528593952215105)
--(axis cs:181000,0.0463230871396922)
--(axis cs:164000,0.0722645365234746)
--(axis cs:137000,0.0563586255171245)
--(axis cs:93000,0.0357172917662236)
--(axis cs:48000,0.0412283072733281)
--(axis cs:20000,0.0673088269969549)
--(axis cs:20000,0.0793781599182586)
--(axis cs:20000,0.0793781599182586)
--(axis cs:48000,0.150167314468308)
--(axis cs:93000,0.0765020301257955)
--(axis cs:137000,0.103949405675695)
--(axis cs:164000,0.149375878653722)
--(axis cs:181000,0.0851302329698255)
--(axis cs:375000,0.0683766981213257)
--(axis cs:1112000,0.089941400454291)
--cycle;

\path [fill=color2, fill opacity=0.3]
(axis cs:800000,0.16372823349095)
--(axis cs:800000,0.0867358230603939)
--(axis cs:721900,0.0784749922164404)
--(axis cs:647500,0.0433732733536408)
--(axis cs:576900,0.155539107315962)
--(axis cs:510000,0.0734970355587459)
--(axis cs:447000,0.0949265943614621)
--(axis cs:387700,0.063454164482887)
--(axis cs:332100,-0.0186037841759882)
--(axis cs:280300,0.0769457017010435)
--(axis cs:232300,0.0940932115766575)
--(axis cs:188100,0.0877520965242126)
--(axis cs:147600,0.161517140455186)
--(axis cs:110900,0.0683793762507783)
--(axis cs:78000,0.0641098917674166)
--(axis cs:48800,0.192384816429856)
--(axis cs:23400,0.0719197600308199)
--(axis cs:23400,0.209770553986188)
--(axis cs:23400,0.209770553986188)
--(axis cs:48800,0.20075800896174)
--(axis cs:78000,0.16499257432694)
--(axis cs:110900,0.166799769266731)
--(axis cs:147600,0.197701465900315)
--(axis cs:188100,0.241740893127907)
--(axis cs:232300,0.173040493738732)
--(axis cs:280300,0.157543391052463)
--(axis cs:332100,0.641542160092109)
--(axis cs:387700,0.112260126722818)
--(axis cs:447000,0.200122876167019)
--(axis cs:510000,0.160010882045431)
--(axis cs:576900,0.188574011915054)
--(axis cs:647500,0.206908713145456)
--(axis cs:721900,0.197623069741712)
--(axis cs:800000,0.16372823349095)
--cycle;

\path [fill=color3, fill opacity=0.3]
(axis cs:1112000,0.90054995893622)
--(axis cs:1112000,0.234908802870095)
--(axis cs:375000,0.0716957489588424)
--(axis cs:181000,0.0394506970071132)
--(axis cs:164000,0.193449900225452)
--(axis cs:137000,0.0522948125772647)
--(axis cs:93000,0.0970948092059918)
--(axis cs:48000,0.188285956401517)
--(axis cs:20000,0.0264339647827417)
--(axis cs:20000,0.145920537686891)
--(axis cs:20000,0.145920537686891)
--(axis cs:48000,0.191173392785371)
--(axis cs:93000,0.688901703221456)
--(axis cs:137000,0.172329842639393)
--(axis cs:164000,0.197358518384693)
--(axis cs:181000,0.663564837068279)
--(axis cs:375000,0.149948745598902)
--(axis cs:1112000,0.90054995893622)
--cycle;

\addplot [semithick, color0, mark=square*, mark size=3, mark options={solid}]
table {%
800000 0.0542203913735065
721900 0.0687162050498995
647500 0.0902277974730955
576900 0.0746910839193464
510000 0.0733636759000461
447000 0.056267131746784
387700 0.0908799762535429
332100 0.0777610940482877
280300 0.0534141763371763
232300 0.0963302807191013
188100 0.072562822657043
147600 0.0510257380163522
110900 0.0774513749468624
78000 0.0803446624320229
48800 0.0796836975501225
23400 0.0673353880242969
};
\addlegendentry{DenseNet AMCE}
\addplot [semithick, color1, mark=*, mark size=4, mark options={solid}]
table {%
1112000 0.078132098467569
375000 0.0606180466714181
181000 0.0657266600547588
164000 0.110820207588598
137000 0.0801540155964098
93000 0.0561096609460095
48000 0.0956978108708178
20000 0.0733434934576068
};
\addlegendentry{WideResNet AMCE}
\addplot [semithick, color2, mark=+, mark size=4, mark options={solid}]
table {%
800000 0.125232028275672
721900 0.138049030979076
647500 0.125140993249548
576900 0.172056559615508
510000 0.116753958802089
447000 0.147524735264241
387700 0.0878571456028525
332100 0.311469187958061
280300 0.117244546376753
232300 0.133566852657695
188100 0.16474649482606
147600 0.17960930317775
110900 0.117589572758754
78000 0.114551233047178
48800 0.196571412695798
23400 0.140845157008504
};
\addlegendentry{DenseNet MCE}
\addplot [semithick, color3, mark=x, mark size=4, mark options={solid}]
table {%
1112000 0.567729380903157
375000 0.110822247278872
181000 0.351507767037696
164000 0.195404209305072
137000 0.112312327608329
93000 0.392998256213724
48000 0.189729674593444
20000 0.0861772512348164
};
\addlegendentry{WideResNet MCE}
\end{axis}

\end{tikzpicture}}
\end{minipage}
}
\caption{Effect of model complexity on confidence calibration.}
\label{fig:CurveCompareECEMCE}
\vspace{-0.05in}
\end{figure}
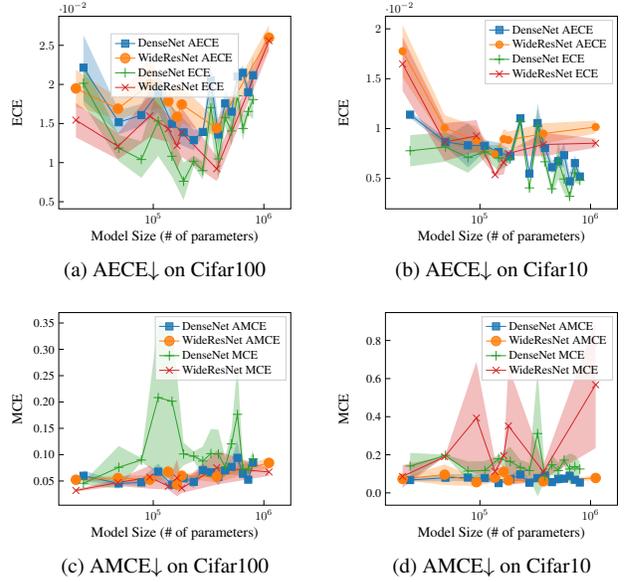

We further validate the evaluation metric on medical image segmentation where uncertainty estimation is crucial. We trained a group of U-Nets~\cite{ronneberger2015u} with standard practice and different widths on the Multi-Modality Whole Heart Segmentation dataset~\cite{zhuang2019evaluation} . The observations are similar to the classification experiments except for a few differences. Firstly, because the network is optimized for better Dice instead of pixel-wise accuracy, the AUPR is not meaningful in this case. 
 Secondly, the difference between ECE and AECE is almost zero. The small size of the dataset make the model trend to overfitting. There are also a lot of background voxels that are easy to predict. These two factors lead to the fact that most predictions have a confidence score close to 1. Then the calibration error at the high confidence area dominate the ECE, despite different binning strategy. However, MCE is still very unstable. This also validates that an excessive number of samples along cannot solve the issues of binning strategy. Detailed results are shown in the supplementary.

\section{Conclusions}
Understanding the quality of uncertainty estimation is critical when applying DNNs to real-world vision problems. We focus on two main use cases of uncertainty estimation, {\em i.e.}, selective prediction and confidence calibration. We identified the issues with the existing metrics for uncertainty estimation that may lead to unreliable or misleading results, and proposed new justified metrics to mitigate these issues. Finally, we validated the new metrics by exploring the effect of model complexity on uncertainty estimation while showing that selective prediction and confidence calibration have different complexity-uncertainty trade-offs.
{\small
\bibliographystyle{ieee_fullname}
\bibliography{egbib}
}

\appendix

\section*{A. Examples of Adaptive Binning}

Figure~\ref{fig:example} shows the Reliability Diagram of DenseNet on Cifar10 and Cifar100 datasets before and after model calibration using adaptive binning. It can be seen that the calibration with temperature scaling significantly reduces the calibration error. For a more difficult dataset and a calibrated model, more bins are used automatically.

\begin{figure}[htb]
\vspace{-0.15in}
\centering
%\hspace*{-2.0em}
\subfloat[Uncalibrated Cifar10]{
\begin{minipage}[b]{0.5\linewidth}
\label{fig:RDC1}
\includegraphics[height=1.2in]{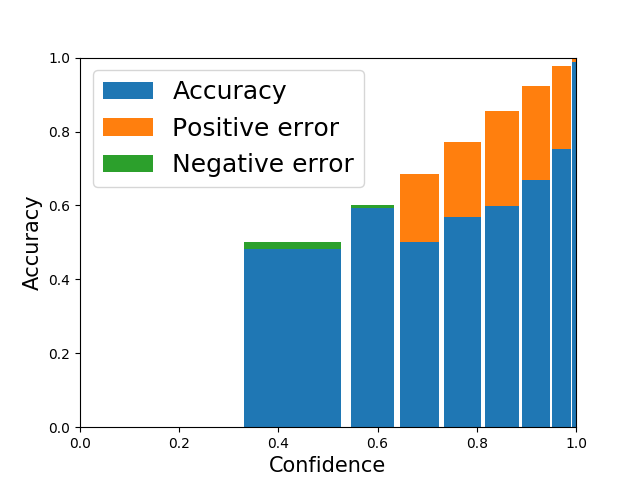}
\end{minipage}
}%\hspace*{-0.3em}
\subfloat[Uncalibrated Cifar100]{
\begin{minipage}[b]{0.5\linewidth}
\label{fig:RDC2}
\includegraphics[height=1.2in]{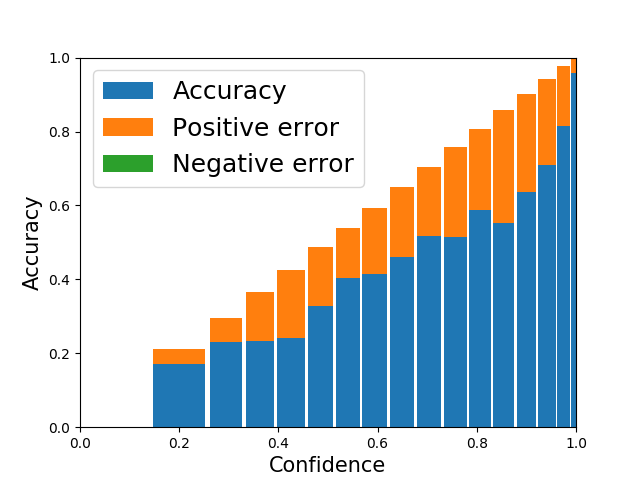}
\end{minipage}
}\\
%\hspace*{-0.3em}
\subfloat[Calibrated Cifar10]{
\begin{minipage}[b]{0.5\linewidth}
\label{fig:RDC3}
\includegraphics[height=1.2in]{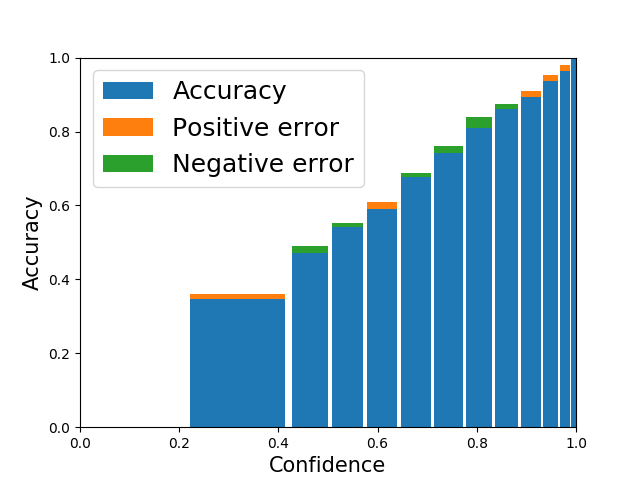}
\end{minipage}
}%\hspace*{-0.3em}
\subfloat[Calibrated Cifar100]{
\begin{minipage}[b]{0.5\linewidth}
\label{fig:RDC4}
\includegraphics[height=1.2in]{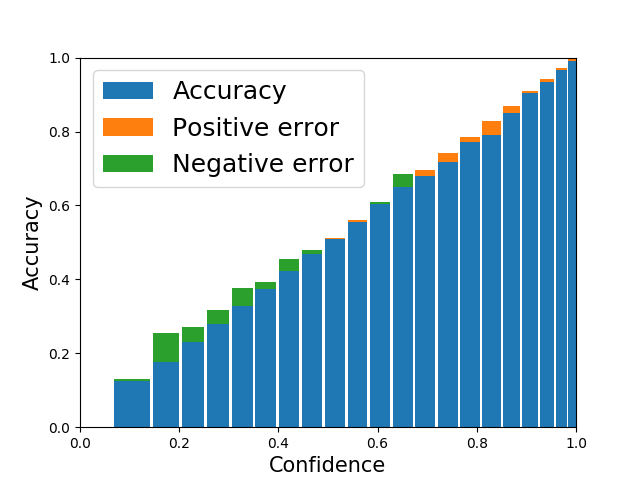}
\end{minipage}
}
%\quad
\caption{Reliability Diagrams of various models.}
\label{fig:example}
\vspace{-0.08in}
\end{figure}

\section*{B Experiment Results of Medical Image Segmentation}
Figure~\ref{fig:seg} shows the experiment results on the Multi-Modality Whole Heart Segmentation dataset.
%\vspace{-0.05in}
\begin{figure}[htb]
%\vspace{-0.12in}
\centering
%\hspace*{-2.2em}
\subfloat[AUPR$\uparrow$]{
\begin{minipage}[b]{0.31\linewidth}
\label{fig:seg1}
\scalebox{0.3}{% This file was created by tikzplotlib v0.8.7.
\begin{tikzpicture}
\pgfplotsset{every axis/.append style={
                    label style={font=\Large}
                    }}
\definecolor{color0}{rgb}{0.12156862745098,0.466666666666667,0.705882352941177}

\begin{axis}[
tick align=inside,
%tick pos=both,
tick pos=left,
x grid style={white!69.01960784313725!black},
xlabel={Model Size (\# of parameters)},
xmin=88048.8, xmax=5862415.2,
xtick style={color=black},
y grid style={white!69.01960784313725!black},
ylabel={AUPR},
ymin=0.996811578028401, ymax=0.997995831739964,
ytick style={color=black},
ytick={0.997,0.9972,0.9974,0.9976,0.9978,0.998},
yticklabels={0.9970,0.9972,0.9974,0.9976,0.9978,0.9980}
]
\path [fill=color0, fill opacity=0.3]
(axis cs:350520,0.99719493414713)
--(axis cs:350520,0.996956316833472)
--(axis cs:788144,0.997352640208802)
--(axis cs:1400680,0.997460136432638)
--(axis cs:2188128,0.997653593774498)
--(axis cs:3150488,0.997629565522482)
--(axis cs:4287760,0.997289760224881)
--(axis cs:5599944,0.997650214826023)
--(axis cs:5599944,0.997677996911501)
--(axis cs:5599944,0.997677996911501)
--(axis cs:4287760,0.997365108318379)
--(axis cs:3150488,0.997851092934893)
--(axis cs:2188128,0.99770979894678)
--(axis cs:1400680,0.997655197659189)
--(axis cs:788144,0.997534230636921)
--(axis cs:350520,0.99719493414713)
--cycle;

\addplot [semithick, color0, mark=*, mark size=3, mark options={solid}]
table {%
350520 0.997075625490301
788144 0.997443435422862
1400680 0.997557667045913
2188128 0.997681696360639
3150488 0.997740329228687
4287760 0.99732743427163
5599944 0.997664105868762
};
\end{axis}

\end{tikzpicture}}
\end{minipage}
}%\hspace*{-0.3em
\subfloat[AURC$\downarrow$]{
\begin{minipage}[b]{0.31\linewidth}
\label{fig:seg2}
\scalebox{0.3}{% This file was created by tikzplotlib v0.8.7.
\begin{tikzpicture}
\pgfplotsset{every axis/.append style={
                    label style={font=\Large}
                    }}
\definecolor{color0}{rgb}{0.12156862745098,0.466666666666667,0.705882352941177}

\begin{axis}[
tick align=inside,
%tick pos=both,
tick pos=left,
x grid style={white!69.01960784313725!black},
xlabel={Model Size (\# of parameters)},
xmin=88048.8, xmax=5862415.2,
xtick style={color=black},
y grid style={white!69.01960784313725!black},
ylabel={AURC},
ymin=0.0262790637767531, ymax=0.0402687953891011,
ytick style={color=black},
ytick={0.026,0.029,0.032,0.035,0.038},
yticklabels={0.026,0.029,0.032,0.035,0.038}
]
\path [fill=color0, fill opacity=0.3]
(axis cs:350520,0.0387238075885398)
--(axis cs:350520,0.0314187322694911)
--(axis cs:788144,0.0278240515773144)
--(axis cs:1400680,0.0329681649048178)
--(axis cs:2188128,0.0324383966612057)
--(axis cs:3150488,0.029763600180265)
--(axis cs:4287760,0.028797790693624)
--(axis cs:5599944,0.0306191167027893)
--(axis cs:5599944,0.0321573593941037)
--(axis cs:5599944,0.0321573593941037)
--(axis cs:4287760,0.029658154428302)
--(axis cs:3150488,0.0317274916027761)
--(axis cs:2188128,0.033526435898197)
--(axis cs:1400680,0.0339352991749568)
--(axis cs:788144,0.0300854749324882)
--(axis cs:350520,0.0387238075885398)
--cycle;

\addplot [semithick, color0, mark=*, mark size=3, mark options={solid}]
table {%
350520 0.0350712699290154
788144 0.0289547632549013
1400680 0.0334517320398873
2188128 0.0329824162797013
3150488 0.0307455458915206
4287760 0.029227972560963
5599944 0.0313882380484465
};
\end{axis}

\end{tikzpicture}}
\end{minipage}
}
\subfloat[Dice$\uparrow$]{
\begin{minipage}[b]{0.31\linewidth}
\label{fig:seg3}
\scalebox{0.3}{% This file was created by tikzplotlib v0.8.7.
\begin{tikzpicture}
\pgfplotsset{every axis/.append style={
                    label style={font=\Large}
                    }}
\definecolor{color0}{rgb}{0.12156862745098,0.466666666666667,0.705882352941177}

\begin{axis}[
tick align=inside,
%tick pos=both,
tick pos=left,
x grid style={white!69.01960784313725!black},
xlabel={Model Size (\# of parameters)},
xmin=88048.8, xmax=5862415.2,
xtick style={color=black},
y grid style={white!69.01960784313725!black},
ylabel={Dice},
ymin=0.852289422374407, ymax=0.904028733965217,
ytick style={color=black},
ytick={0.85,0.86,0.87,0.88,0.89,0.9,0.91},
yticklabels={0.85,0.86,0.87,0.88,0.89,0.90,0.91}
]
\path [fill=color0, fill opacity=0.3]
(axis cs:350520,0.864621693278362)
--(axis cs:350520,0.854641209264899)
--(axis cs:788144,0.883158175544008)
--(axis cs:1400680,0.886927243374074)
--(axis cs:2188128,0.889350311918529)
--(axis cs:3150488,0.88977888322887)
--(axis cs:4287760,0.893088516515306)
--(axis cs:5599944,0.898010412239448)
--(axis cs:5599944,0.901676947074725)
--(axis cs:5599944,0.901676947074725)
--(axis cs:4287760,0.894727861113658)
--(axis cs:3150488,0.89389714891118)
--(axis cs:2188128,0.895556570084939)
--(axis cs:1400680,0.888647631531857)
--(axis cs:788144,0.887107906310526)
--(axis cs:350520,0.864621693278362)
--cycle;

\addplot [semithick, color0, mark=*, mark size=3, mark options={solid}]
table {%
350520 0.85963145127163
788144 0.885133040927267
1400680 0.887787437452966
2188128 0.892453441001734
3150488 0.891838016070025
4287760 0.893908188814482
5599944 0.899843679657086
};
\end{axis}

\end{tikzpicture}}
\end{minipage}
}%\hspace*{-0.3em}
%\vspace{-0.02in}
\\%\hspace*{-0.3em}
\subfloat[AECE$\downarrow$]{
\begin{minipage}[b]{0.5\linewidth}
\label{fig:seg4}
\scalebox{0.45}{% This file was created by tikzplotlib v0.8.7.
\begin{tikzpicture}
\pgfplotsset{every axis/.append style={
                    label style={font=\large}
                    }}
\definecolor{color0}{rgb}{0.12156862745098,0.466666666666667,0.705882352941177}
\definecolor{color1}{rgb}{1,0.498039215686275,0.0549019607843137}

\begin{axis}[
legend cell align={left},
legend style={fill opacity=0.8, draw opacity=1, text opacity=1, at={(0.97,0.03)}, anchor=south east, draw=white!80.0!black},
tick align=inside,
%tick pos=both,
tick pos=left,
x grid style={white!69.01960784313725!black},
xlabel={Model Size (\# of parameters)},
xmin=88048.8, xmax=5862415.2,
xtick style={color=black},
y grid style={white!69.01960784313725!black},
ylabel={ECE},
ymin=0.0174704727759708, ymax=0.027895657806152,
ytick style={color=black},
ytick={0.018,0.02,0.022,0.024,0.026,0.028},
yticklabels={0.019,0.020,0.021,0.022,0.023,0.024,0.025,0.026}
]
\path [fill=color0, fill opacity=0.3]
(axis cs:350520,0.0209908642317867)
--(axis cs:350520,0.0197625266409791)
--(axis cs:788144,0.0216914376763375)
--(axis cs:1400680,0.0241455425119383)
--(axis cs:2188128,0.0238098205728398)
--(axis cs:3150488,0.0237825267069022)
--(axis cs:4287760,0.0253340215456406)
--(axis cs:5599944,0.0244418203001807)
--(axis cs:5599944,0.0255692157160584)
--(axis cs:5599944,0.0255692157160584)
--(axis cs:4287760,0.0254140963353496)
--(axis cs:3150488,0.0249917172583434)
--(axis cs:2188128,0.0246817707190097)
--(axis cs:1400680,0.0256035110697398)
--(axis cs:788144,0.023089681150889)
--(axis cs:350520,0.0209908642317867)
--cycle;

\path [fill=color1, fill opacity=0.3]
(axis cs:350520,0.0209908642317867)
--(axis cs:350520,0.019762526640979)
--(axis cs:788144,0.0216858158266138)
--(axis cs:1400680,0.0241382655230508)
--(axis cs:2188128,0.023806459427539)
--(axis cs:3150488,0.0237825267069021)
--(axis cs:4287760,0.0253340760649388)
--(axis cs:5599944,0.0244381572171096)
--(axis cs:5599944,0.0255679772972287)
--(axis cs:5599944,0.0255679772972287)
--(axis cs:4287760,0.0254116234866166)
--(axis cs:3150488,0.0249917172583434)
--(axis cs:2188128,0.024680347488426)
--(axis cs:1400680,0.0256036039411437)
--(axis cs:788144,0.0230890526789988)
--(axis cs:350520,0.0209908642317867)
--cycle;

\addplot [semithick, color0, mark=square*, mark size=3, mark options={solid}]
table {%
350520 0.0203766954363829
788144 0.0223905594136132
1400680 0.024874526790839
2188128 0.0242457956459248
3150488 0.0243871219826228
4287760 0.0253740589404951
5599944 0.0250055180081195
};
\addlegendentry{ECE}
\addplot [semithick, color1, mark=*, mark size=3, mark options={solid}]
table {%
350520 0.0203766954363829
788144 0.0223874342528063
1400680 0.0248709347320973
2188128 0.0242434034579825
3150488 0.0243871219826228
4287760 0.0253728497757777
5599944 0.0250030672571692
};
\addlegendentry{AECE}
\end{axis}

\end{tikzpicture}}
\end{minipage}
}%\hspace*{-0.3em}
\subfloat[AMCE$\downarrow$]{
\begin{minipage}[b]{0.5\linewidth}
\label{fig:seg5}
\scalebox{0.45}{% This file was created by tikzplotlib v0.8.7.
\begin{tikzpicture}

\pgfplotsset{every axis/.append style={
                    label style={font=\large}
                    }}

\definecolor{color0}{rgb}{0.12156862745098,0.466666666666667,0.705882352941177}
\definecolor{color1}{rgb}{1,0.498039215686275,0.0549019607843137}

\begin{axis}[
legend cell align={left},
legend style={fill opacity=0.8, draw opacity=1, text opacity=1, draw=white!80.0!black},
tick align=inside,
%tick pos=both,
tick pos=left,
x grid style={white!69.01960784313725!black},
xlabel={Model Size (\# of parameters)},
xmin=88048.8, xmax=5862415.2,
xtick style={color=black},
y grid style={white!69.01960784313725!black},
ylabel={MCE},
ymin=0.0968100176512658, ymax=0.637987506428484,
ytick style={color=black},
ytick={0,0.1,0.2,0.3,0.4,0.5,0.6,0.7},
yticklabels={0.0,0.1,0.2,0.3,0.4,0.5,0.6,0.7}
]
\path [fill=color0, fill opacity=0.3]
(axis cs:350520,0.193867837513693)
--(axis cs:350520,0.142828234015276)
--(axis cs:788144,0.160926542029524)
--(axis cs:1400680,0.189777058328588)
--(axis cs:2188128,0.194284144892327)
--(axis cs:3150488,0.188616954310799)
--(axis cs:4287760,0.20086933479221)
--(axis cs:5599944,0.153148622708443)
--(axis cs:5599944,0.613388529665883)
--(axis cs:5599944,0.613388529665883)
--(axis cs:4287760,0.280096173821006)
--(axis cs:3150488,0.273803346860137)
--(axis cs:2188128,0.60667376334068)
--(axis cs:1400680,0.201354585693692)
--(axis cs:788144,0.256714841963323)
--(axis cs:350520,0.193867837513693)
--cycle;

\path [fill=color1, fill opacity=0.3]
(axis cs:350520,0.132325831856567)
--(axis cs:350520,0.121408994413867)
--(axis cs:788144,0.155912141836818)
--(axis cs:1400680,0.195621761069334)
--(axis cs:2188128,0.190090601604198)
--(axis cs:3150488,0.205421703808857)
--(axis cs:4287760,0.221391892802816)
--(axis cs:5599944,0.214127875331957)
--(axis cs:5599944,0.229965041225346)
--(axis cs:5599944,0.229965041225346)
--(axis cs:4287760,0.221847878639084)
--(axis cs:3150488,0.214888986638051)
--(axis cs:2188128,0.202116168104535)
--(axis cs:1400680,0.212814368991318)
--(axis cs:788144,0.16896635467426)
--(axis cs:350520,0.132325831856567)
--cycle;

\addplot [semithick, color0, mark=square*, mark size=3, mark options={solid}]
table {%
350520 0.168348035764484
788144 0.208820691996423
1400680 0.19556582201114
2188128 0.400478954116503
3150488 0.231210150585468
4287760 0.240482754306608
5599944 0.383268576187163
};
\addlegendentry{MCE}
\addplot [semithick, color1, mark=*, mark size=3, mark options={solid}]
table {%
350520 0.126867413135217
788144 0.162439248255539
1400680 0.204218065030326
2188128 0.196103384854366
3150488 0.210155345223454
4287760 0.22161988572095
5599944 0.222046458278652
};
\addlegendentry{AMCE}
\end{axis}

\end{tikzpicture}}
\end{minipage}
}
\caption{Effect of model complexity on uncertainty estimation in medical image segmentation.}
\label{fig:seg}
%\vspace{-0.17in}
\end{figure}
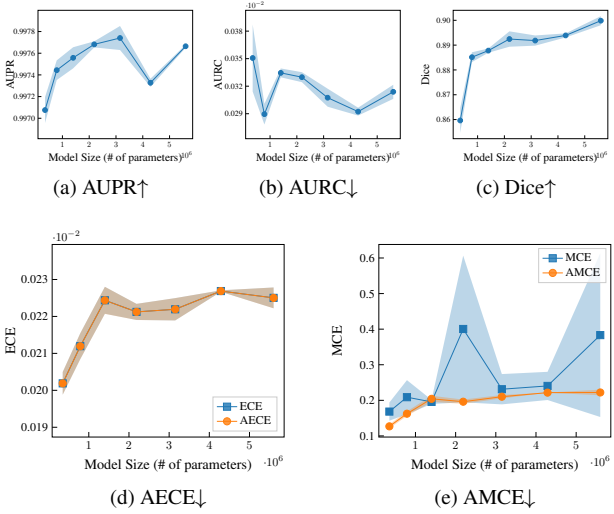

\section*{C. Proofs}
%\appendix

\subsection*{C.1 Proof of Theorem 1}
  \begin{theorem}
\label{T2}
  For any two networks A and B of the same accuracy and their uncertainties measured by arbitrary methods (which can be different for A and B), the curve of A dominates that of B in the ROC space if and only if the curve of A dominates that of B in the Risk-Coverage space. \end{theorem}
  \begin{proof}
  Proof by contradiction. Since the wrong predictions are positive samples and correct predictions are negative samples. Having the same accuracy means that the two networks have the same number of positive and negative samples. Denote TN, TP, FN, FP, TPR, and FPR as true negative, true positive, false negative, false positive, true positive rate, and false positive rate respectively. Then we have
%   \small
% \begin{align}
% \label{equ:ECE}
%     coverage =\frac{1}{|D|}\sum_{k=1}^{m}|D_k||E_{P_{\theta,\mathcal{D}}(c|r_k)}[c]-r_k|
% \end{align}
% \normalsize
   \small
\begin{align}
    coverage=\frac{TN+FN}{TN+FN+TP+FP}
\end{align}
\begin{align}
    risk = \frac{FN}{TN+FN}
\end{align}
\normalsize
  
  Suppose the curve of B dominates that of A in the ROC space but not in the Risk-Coverage space. Then there exists a point $a$ on the curve of network A and a point $b$ on the curve of B such that $coverage_a=coverage_b$ and $risk_a<risk_b$.
  
  From $coverage_a=coverage_b$, we have $TN_a+FN_a=TN_b+FN_b$. Since $\frac{FN_a}{(TN_a+FN_a)}<\frac{FN_b}{(TN_b+FN_b)}$, we have $FN_a<FN_b$ and $TN_a>TN_b$. 
  
  Remember that the numbers of positive and negative samples are equal. Therefore we have $FP_a+TN_a=FP_b+TN_b$ and $TP_a+FN_a=TP_b+FN_b$. Then we obtain $FP_a<FP_b$ and $TP_a>TP_b$. Then we have $TPR_a>TPR_b$ and $FPR_a<FPR_b$.
  
  This contradicts the fact that the curve of B is higher than that of A in the ROC space. The other direction can be proved in the same way.
  \end{proof}

\subsection*{C.2 Proof of Proposition 1}
\begin{proposition}
\label{T1}
  For any bin selection, $\hat{\text{ECE}}(P_{\theta,\mathcal{D}})=\text{ECE}(P_{\theta,\mathcal{D}})$ if and only if for any bin $B_j$,  $E_{P_{\theta,\mathcal{D}}(c|r_k)}[c]\geq r_k$ for all $r_k\in B_j$ or $E_{P_{\theta,\mathcal{D}}(c|r_k)}[c]\leq r_k$ for all $r_k\in B_j$. Otherwise, $\hat{\text{ECE}}(P_{\theta,\mathcal{D}})<\text{ECE}(P_{\theta,\mathcal{D}})$.
  \end{proposition}
  \begin{proof}
  %the samples are enough, and 
  For clarity, we reuse $n$, $B_j$ and $D_j$ as the number of bins, the range of bin, and the sample set where $j=\{1,\dots,n\}$. Note that here the bin selection no longer needs to be a uniform partition.
In order to make $\text{ECE}(P_{\theta,\mathcal{D}})$ meaningful, we assume there are enough samples and $E_{P_{\theta,\mathcal{D}}(c|r_k)}[c]$ is a solvable value. 
Denote the number of different values of $r$ as $m$ and these different values of $r$ as $r_k$ where $k=\{1,\dots,m\}$. Then we partition $D$ to $m$ bins $D_k=\{x_i|r_i=r_k\}$ so that each bin only has one unique $r$ value.
The ground truth ECE can be written as:
 \small
\begin{align}
\label{equ:ECE}
    \text{ECE}(P_{\theta,\mathcal{D}}) =\frac{1}{|D|}\sum_{k=1}^{m}|D_k||E_{P_{\theta,\mathcal{D}}(c|r_k)}[c]-r_k|
\end{align}
\normalsize
  Then we have 
    \small{
  \begin{align}
    \hat{\text{ECE}}(P_{\theta,\mathcal{D}}) &=\frac{1}{|D|}\sum_{j=0}^{n}|\sum_{x_i\in D_j}c_i-\sum_{x_i\in D_j}r_i|\\\nonumber
    &=\frac{1}{|D|}\sum_{j=0}^{n}|\sum_{r_k\in B_j}|D_k|(E_{P_{\theta,\mathcal{D}}(c|r_k)}[c]-r_k)|\nonumber 
\end{align}  
}
\normalsize
Note that
\small{
  \begin{align}
  \sum_{r_k\in B_j}|D_k|(E_{P_{\theta,\mathcal{D}}(c|r_k)}[c]-r_k)\leq \sum_{r_k\in B_j}|D_k|\left|(E_{P_{\theta,\mathcal{D}}(c|r_k)}[c]-r_k)\right|\nonumber
\end{align}  
}\normalsize
and they are equal if and only if for any bin $B_j$,  $E_{P_{\theta,\mathcal{D}}(c|r_k)}[c]\geq r_k$ for all $r_k\in B_j$ or $E_{P_{\theta,\mathcal{D}}(c|r_k)}[c]\leq r_k$ for all $r_k\in B_j$. 
Together with Equation~\ref{equ:ECE}, we conclude the proof.
\end{proof}

\subsection*{C.3 Proof of Proposition 2}
\begin{proposition}
\label{TU}
The uncertainty estimation $r$ is perfect for both selective prediction and confidence calibration if and only if, for all samples $r\in\{0,1\}$, $E_{P_{\theta,\mathcal{D}}(c|r=0)}[c]=0$, and $E_{P_{\theta,\mathcal{D}}(c|r=1)}[c]=1$.
  \end{proposition}
\begin{proof}
Given $r\in\{0,1\}$, $E_{P_{\theta,\mathcal{D}}(c|r=0)}[c]=0$, and $E_{P_{\theta,\mathcal{D}}(c|r=1)}[c]=1$, it follows trivially that $E_{P_{\theta,\mathcal{D}}(r)}[|E_{P_{\theta,\mathcal{D}}(c|r)}[c]-r|]=0$ and $r_a>r_b$ for any $c_a=1$, $c_b=0$. 

On the other side, if $E_{P_{\theta,\mathcal{D}}(r)}[|E_{P_{\theta,\mathcal{D}}(c|r)}[c]-r|]=0$, we have $E_{P_{\theta,\mathcal{D}}(c|r)}[c]=r$. If there exists a $x_i$ that $r_i\in(0,1)$, then we $E_{P_{\theta,\mathcal{D}}(c|r_i)}[c]\in(0,1)$. 

Consequently, $|\{x|c=0,r=r_i\}|>0$ and $|\{x|c=1,r=r_i\}|>0$. Then for any two samples $x_a\in \{x|c=1,r=r_i\}$ and $x_b\in\{x|c=0,r=r_i\}$, we have $c_a=1$, $c_b=0$ and $r_a=r_b$ that contradict with the fact that $r_a>r_b$ for $c_a=1$, $c_b=0$. Therefore, $x_i\in\{0,1\}$. Using $E_{P_{\theta,\mathcal{D}}(c|r)}[c]=r$, it follows immediately that $E_{P_{\theta,\mathcal{D}}(c|r=0)}[c]=0$, and $E_{P_{\theta,\mathcal{D}}(c|r=1)}[c]=1$.
\end{proof}

%\bibliographystyle{aaai}
%\bibliography{egbib}

\end{document}